\pdfoutput=1

\documentclass{article}
\usepackage{amd}

\newcommand{\AMDlogo}{\includegraphics[height=0.28in]{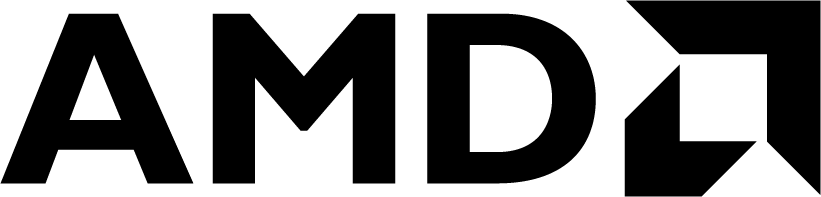}}

\usepackage{amsmath,amsfonts,bm}

\def\secref#1{section~\ref{#1}}
\def\Secref#1{Section~\ref{#1}}
\def\eqref#1{equation~\ref{#1}}
\def\1{\bm{1}}

\DeclareMathAlphabet{\mathsfit}{\encodingdefault}{\sfdefault}{m}{sl}
\SetMathAlphabet{\mathsfit}{bold}{\encodingdefault}{\sfdefault}{bx}{n}

\usepackage{hyperref}
\usepackage{url}
\usepackage{booktabs}
\usepackage{amsthm}
\usepackage{amssymb}
\usepackage{float}
\usepackage{pgfplots}
\pgfplotsset{compat=1.18}
\usepgfplotslibrary{groupplots}
\usetikzlibrary{arrows.meta,positioning,calc}
\definecolor{tgreen}{HTML}{228A49}
\definecolor{tblue}{HTML}{2575FC}
\definecolor{tgold}{HTML}{8A6410}
\newcommand{\corrsymbol}{{\boldmath$\ddagger$}}

\newtheorem{proposition}{Proposition}
\newtheorem{corollary}{Corollary}

\newcommand{\Pmax}{P_{\mathrm{max}}}
\newcommand{\ptrunct}{p_{\mathrm{trunc},t}}

\begin{document}
\onecolumn
\amdtitlerunning{Long-Context Drafting Does Not Require Long-Context Drafters}

\amdtitle{Long-Context Drafting Does Not Require\\ Long-Context Drafters}

\vspace*{-0.5\baselineskip}
{\centering\setlength{\parskip}{0pt}%
  {\bfseries Aref Jafari\textsuperscript{1,\corrsymbol}\qquad Vansh Bhatia\textsuperscript{1}\qquad Matt Qin\textsuperscript{1}\qquad Subhajit Dutta Chowdhury\textsuperscript{1}\qquad Ben Lee\textsuperscript{1}\\ Vikram Appia\textsuperscript{1}\qquad Ali Ghodsi\textsuperscript{2,\corrsymbol}\qquad Mehdi Rezagholizadeh\textsuperscript{1,\corrsymbol} \par}
  \vspace{0.75em}
  \textsuperscript{1}Advanced Micro Devices Inc. (AMD) ~~
  \textsuperscript{2}University of Waterloo \par
   \vspace{0.15em}
    \textsuperscript{\corrsymbol}\,Correspondence to:
    \texttt{\char`\{aref.jafari, mehdi.rezagholizadeh\char`\}@amd.com} \\
    \texttt{ali.ghodsi@uwaterloo.ca}\par}

\amdkeywords{Speculative decoding, long-context inference, rotary position embeddings, LLM serving}
\vskip 0.15in

\begin{abstract}
Speculative decoding pairs a cheap drafter with parallel verification against the target.
At ultra-long context it stops paying, because a drafter that attends over the full prefix inherits the target's context-dependent cost, so the speedup it was meant to provide first erodes and then reverses.
The standard response is to build a long-context drafter, by training or fine-tuning one, or by using a multi-token-prediction head shipped inside the target checkpoint.
We argue this is unnecessary, because the drafters in use are \emph{target-conditioned}, reading hidden states the target computed over the entire prefix, so distant context already reaches them without their own attention, and verification corrects what they miss.
We therefore freeze released short-context drafters and run them far past their training length under two inference-time mechanisms, one bounding the attention cache they see, which fixes their cost, and one re-basing their rotary position indices into the trained range.
No weight is touched and the only tunable saturates early, so one {vLLM}-side configuration spans 4k to two million tokens, costing at most $0.16\times$ of speedup at short length.
We prove when a common index shift leaves the drafter's computation exactly unchanged, and decompose acceptance into two measurable gaps with sharp bounds.
Whatever speedup a frozen drafter delivers at short context, it retains out to two million tokens, across three target–drafter pairs and five corpora. The claim concerns the flatness rather than the multiple, and removing either mechanism eliminates it. The gain is conditional on the speculative cost per emitted token staying below the matched target-only cost.
\end{abstract}

\section{Introduction}
\label{sec:intro}

Speculative decoding~\citep{leviathan2023speculative,chen2023accelerating} reduces generation latency and is designed to leave the target's output distribution unchanged, a drafter proposing $K$ tokens that the target verifies in one forward pass and corrects itself where it rejects.
The speedup depends on how often proposals are accepted and on how cheap drafting is relative to a target decode step, and at ultra-long context the second quantity is where the method fails.
Writing $N$ for the context length, a drafter attending over a growing prefix pays attention cost scaling in $N$ exactly as the target does, so the ratio of draft to target cost does not improve with scale, and at million-token contexts a full-attention drafter is not a cheap approximation of the target but a second expensive model.  The asymmetry is in memory traffic, a $K$-token draft performing $K$ sequential reads of a {KV} cache of size $N$ where verification performs one parallel read of the same cache, so drafting pays that bandwidth cost $K$ times per round and verification once, which is why the drafter is the side worth bounding.
The field's response has been to build drafters suited to long context, by training them~\citep{yang2026longspec,wang2026specla,li2024eagle} or by using a native multi-token-prediction head that only checkpoints shipping such a head can use~\citep{valliappan2026windowed}; each either requires training or requires the target to have been built with drafting in mind.

We argue the premise is wrong: \textbf{long-context drafting does not require long-context drafters.}
The drafters at issue are \emph{target-conditioned}: in the {EAGLE} family~\citep{li2024eagle,li2025eagle3} and in {DSpark}~\citep{cheng2026dspark} the drafter's per-position input concatenates a token embedding with hidden states the target computed over the entire prefix, so distant context reaches it through those features rather than through its own attention, and bounding how many positions it attends over bounds its cost without cutting it off from the prefix.  Verification makes this safe rather than making it work, a guess that does depend on distant information being rejected at a bounded cost of one round.
Running such a drafter far past its training length fails in two separately fixable ways, shown together in Figure~\ref{fig:overview}.  Its attention \textbf{cost} grows with the prefix, so we bound it to the most recent $W$ tokens and leave the target verifying against the complete context.  Windowing does not bound the \textbf{position} indices, so a rotary drafter still evaluates its transform on indices of order $N$, beyond the finite reach $\Pmax$ its implementation was built for; that failure is quiet, degrading acceptance while a correct verifier leaves the target's intended output law unchanged, and we remove it by re-basing indices into $[0,\Pmax)$, a common shift within an index band whose output equivalence holds under Proposition~\ref{prop:invariance}.  The two are not substitutes, windowing alone leaving indices to grow with context and re-basing alone placing a growing dependency graph across bands where no single shift preserves every offset.  Both are weight-preserving and both are implemented as a patch to {vLLM}~\citep{kwon2023vllm}, the engine in which every result below is measured.

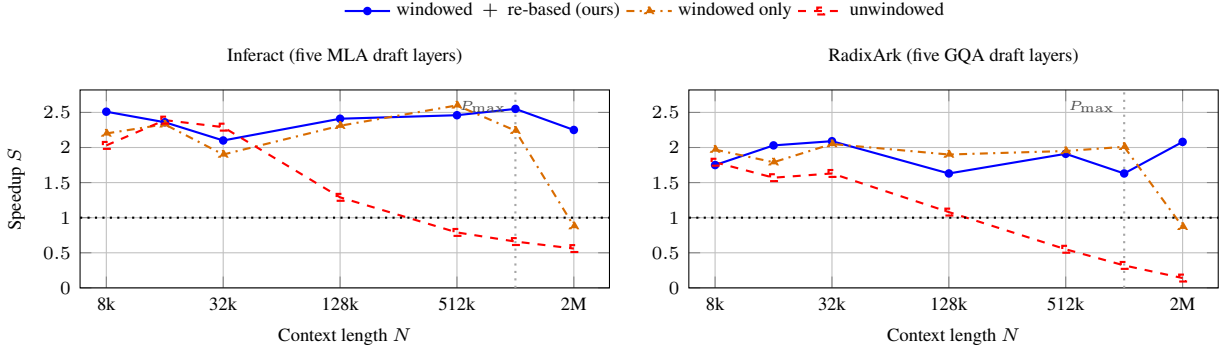
\begin{figure}[t]
\centering
\begin{tikzpicture}
\begin{groupplot}[
  group style={group size=2 by 1, horizontal sep=1.05cm},
  width=.52\linewidth, height=4.2cm,
  xmode=log, log basis x=2,
  xlabel={Context length $N$},
  xmin=6000,xmax=3200000,ymin=0,ymax=2.82,
  xtick={8192,32768,131072,524288,2097152},
  xticklabels={8k,32k,128k,512k,2M},
  ytick={0,0.5,1,1.5,2,2.5},
  xticklabel style={font=\scriptsize},
  yticklabel style={font=\scriptsize},
  label style={font=\scriptsize},
  title style={font=\scriptsize},
  grid=major,
]
\nextgroupplot[ylabel={Speedup $S$},title={{Inferact}} (five {MLA} draft layers),
  legend style={at={(1.07,1.30)},anchor=south,font=\scriptsize,legend columns=3,draw=none,fill=none}]
\addplot[blue,thick,mark=*,mark size=1.2pt] coordinates
{(8192,2.51)(16384,2.36)(32768,2.1)(131072,2.41)(524288,2.46)(1048576,2.55)(2097152,2.25)};
\addlegendentry{windowed ${}+{}$ re-based (ours)}
\addplot[orange!85!black,thick,dashdotted,mark=triangle*,mark size=1.6pt] coordinates
{(8192,2.2)(16384,2.33)(32768,1.9)(131072,2.31)(524288,2.6)(1048576,2.24)(2097152,0.88)};
\addlegendentry{windowed only}
\addplot[red,thick,dashed,mark=square,mark size=1.3pt] coordinates
{(8192,2.03)(16384,2.39)(32768,2.29)(131072,1.29)(524288,0.79)(1048576,0.66)(2097152,0.56)};
\addlegendentry{unwindowed}
\addplot[black,dotted,thick,forget plot] coordinates {(6000,1)(3200000,1)};
\addplot[gray!70,thick,dotted,forget plot] coordinates {(1048576,0)(1048576,2.82)};
\node[font=\tiny,gray!60!black,anchor=north east] at (axis cs:1048576,2.80) {$\Pmax$};
\nextgroupplot[title={{RadixArk}} (five {GQA} draft layers)]
\addplot[blue,thick,mark=*,mark size=1.2pt] coordinates
{(8192,1.75)(16384,2.03)(32768,2.09)(131072,1.63)(524288,1.91)(1048576,1.63)(2097152,2.08)};
\addplot[orange!85!black,thick,dashdotted,mark=triangle*,mark size=1.6pt] coordinates
{(8192,1.97)(16384,1.79)(32768,2.05)(131072,1.9)(524288,1.95)(1048576,2.01)(2097152,0.87)};
\addplot[red,thick,dashed,mark=square,mark size=1.3pt] coordinates
{(8192,1.79)(16384,1.57)(32768,1.63)(131072,1.08)(524288,0.55)(1048576,0.32)(2097152,0.14)};
\addplot[black,dotted,thick,forget plot] coordinates {(6000,1)(3200000,1)};
\addplot[gray!70,thick,dotted,forget plot] coordinates {(1048576,0)(1048576,2.82)};
\node[font=\tiny,gray!60!black,anchor=north east] at (axis cs:1048576,2.80) {$\Pmax$};
\end{groupplot}
\end{tikzpicture}
\caption{Speedup against context length, $\log_2$ axis, both drafters and all three arms (\secref{sec:e4}; ladder in Table~\ref{tab:e4}).  Dotted horizontal is break-even; dotted vertical is $\Pmax=1{,}048{,}576$, the rotary limit both drafters reach by {YaRN} scaling.  Only the full method (blue) is still worth at 2M roughly what it is worth at 8k.}
\label{fig:e4}
\end{figure}

\begin{figure}[t]
\centering
\begin{tikzpicture}[
  font=\scriptsize,
  blk/.style={rounded corners=3pt, draw, semithick, align=center, inner sep=3.5pt},
  pill/.style={rounded corners=5pt, draw, semithick, fill=white, align=center,
               inner xsep=4pt, inner ysep=2.5pt},
  tok/.style={rounded corners=1.5pt, draw=black!40, fill=white,
              minimum height=10pt, inner xsep=2pt, inner ysep=0pt},
  wtok/.style={tok, draw=tgreen, text=tgreen},
  ar/.style={-{Stealth[length=4.5pt,width=3.5pt]}, semithick},
  flow/.style={ar, dash pattern=on 2.6pt off 2pt},
]

% ================ token timeline ================
\node[rounded corners=3pt, draw=black!35, fill=black!4, anchor=south west,
      minimum width=0.545\textwidth, minimum height=18pt] (distbox) at (0pt,0pt) {};
\node[rounded corners=3pt, draw=tgreen, semithick, fill=tgreen!10, anchor=south east,
      minimum width=0.425\textwidth, minimum height=18pt] (winbox) at (\textwidth,0pt) {};

\node[tok, anchor=west] (k1) at ([xshift=6pt]distbox.west) {$x_1$};
\node[tok, right=14pt of k1] (k2) {$x_2$};
\node[right=15pt of k2] (kd1) {$\cdots$};
\node[tok, right=15pt of kd1] (km) {$x_m$};
\node[right=15pt of km] (kd2) {$\cdots$};
\node[tok, right=15pt of kd2] (kw) {$x_{t-W}$};

\node[wtok, anchor=east] (v3) at ([xshift=-6pt]winbox.east) {$x_t$};
\node[wtok, left=18pt of v3] (v2) {$x_{t-1}$};
\node[left=18pt of v2, text=tgreen] (vd) {$\cdots$};
\node[wtok, left=18pt of vd] (v1) {$x_{t-W+1}$};

\node[anchor=north, text=black!55] at ([yshift=-2pt]distbox.south)
      {not attended by the drafter};
\node[anchor=north, text=tgreen] at ([yshift=-2pt]winbox.south)
      {the $W$ tokens the drafter attends over};

% ================ nested spans ================
\draw[tgreen, semithick, rounded corners=1.5pt]
      (0pt,18pt) -- (0pt,24pt) -- (\textwidth,24pt) -- (\textwidth,18pt);
\draw[tblue, semithick, rounded corners=1.5pt]
      ([yshift=9pt]winbox.north west) -- ([yshift=16pt]winbox.north west)
   -- ([yshift=16pt]winbox.north east) -- ([yshift=9pt]winbox.north east);

% ================ model blocks ================
\node[blk, draw=tgreen, fill=tgreen!12, anchor=south west] (tgt) at (0pt,92pt)
      {\textbf{Target model}\\{\color{tgreen}\itshape attends $x_1,\dots,x_t$}\\
       cost per step grows with $N$};
\node[blk, draw=tblue, fill=tblue!10, anchor=south east] (drf) at (\textwidth,92pt)
      {\textbf{Drafter}, frozen\\{\color{tblue}\itshape attends $W$ tokens}\\
       cost per step constant in $N$};

% ================ mechanism 2: re-basing ================
\node[pill, draw=tgold] (reb) at (0.775\textwidth,58pt)
      {index re-basing\\{\color{tgold}$m\mapsto m\bmod\Pmax$}};

\draw[tblue, flow] (0.775\textwidth,34pt) -- (reb.south);
\draw[tblue, flow] (reb.north) -- (drf.south -| reb.north);
\draw[tgreen, flow] (0.13\textwidth,24pt) -- (tgt.south -| 0.13\textwidth,0pt);

% ================ mechanism 1: feature channel ================
\draw[ar, tblue!75!black, line width=0.9pt] (tgt.east) -- (drf.west);
\node[pill, draw=tblue!75!black] at ($(tgt.east)!0.5!(drf.west)$)
      {target hidden states $h_{\mathrm{target}}$};

% ================ proposal and verification ================
\node[pill, draw=tblue, fill=tblue!10, anchor=east] (out) at (\textwidth,152pt)
      {$K$ draft tokens};
\draw[tblue, ar] (drf.north -| out.south) -- (out.south);
\draw[tgreen, flow] (out.west) -- (0.13\textwidth,152pt) -- (tgt.north -| 0.13\textwidth,0pt);
\node[pill, draw=tgreen] at (0.46\textwidth,152pt) {verified in one target pass};

\end{tikzpicture}
\caption{The method.  Windowing bounds \emph{what the drafter attends over}, holding its cost constant in $N$; re-basing bounds \emph{the indices it evaluates}, which windowing does not, by assigning $m\bmod\Pmax$ on cache entry.}
\label{fig:overview}
\end{figure}
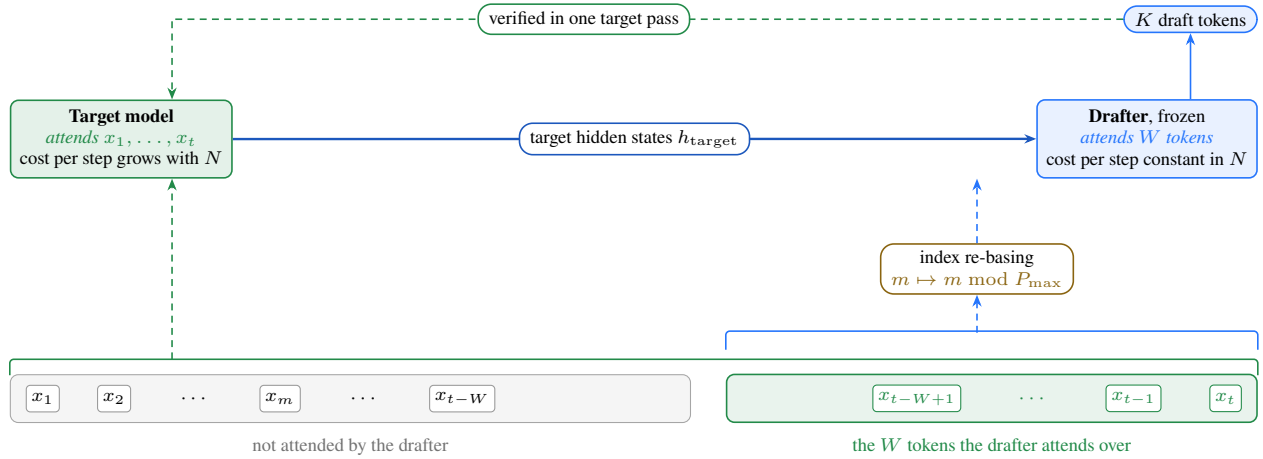

Our contributions are the following.
\textbf{(i)~Theory.}
We identify when a common {RoPE} shift preserves the drafter's computation graph, and give sharp unconditional acceptance bounds through a target truncation gap and a deployed-drafter fidelity gap (\secref{sec:theory}).
\textbf{(ii)~A re-basing scheme with an exactness guarantee.}
When an output's dependency graph lies in one index band, re-basing is a common shift matching ideal true-index {RoPE} in exact arithmetic, with a key's index stable for its lifetime, so nothing cached is re-rotated.
\textbf{(iii)~A measured mechanism.}
We measure both gaps rather than assume either is small (\secref{sec:gaps}), and fidelity is large and flat in $W$, so acceptance is drafter-limited rather than window-limited.
\textbf{(iv)~Evidence at a scale where the problem is visible.}
Two frozen, released drafters retain 90\% and 119\% of their 8k speedup at 2M, against 28\% and 8\% unwindowed and 40\% and 44\% windowed without re-basing (Figure~\ref{fig:e4}, \secref{sec:e4}), reproduced on two further pairs, a dense 7B and a block-sparse target that does not itself slow with context (\secref{sec:e1}), and on four further corpora, the three reaching 1M each retaining $0.93$--$1.30$ of their 8k speedup there (Appendix~\ref{app:datasets}); a condition under which windowing \emph{fails} to pay is in Appendix~\ref{app:more}.
\textbf{(v)~Simplicity, and the reach it buys.}
The recipe needs no weight updates, no new checkpoint, and no sinks, and its one parameter, the window, is stable from 32k to 1M (Appendix~\ref{app:winabl}), so a single configuration spans 4k to 2M at a short-context cost of at most $0.16\times$; we are not aware of prior work demonstrating speculative decoding at multi-million-token context with frozen, released short-context drafters.  Item~(iv) is measured on a different pair from item~(iii), and we do not claim the latter explains the former.

\section{Background}
\label{sec:background}

Given target $p$ and drafter $q$ over vocabulary $\mathcal{V}$, a speculative round proposes $K$ tokens from $q$ and verifies them against $p$ in one parallel pass, with $p_t$ and $q_t$ their predictive distributions at position $t$ and the conditioning made explicit in \secref{sec:decomposition}.
Writing $\mathrm{TV}(u,v):=\tfrac12\|u-v\|_1$ for total variation, rejection sampling~\citep{leviathan2023speculative} gives per-token acceptance probability
\begin{equation}
\label{eq:alpha}
\alpha_t=\sum_{x:q_t(x)>0}q_t(x)\min\left\{1,\frac{p_t(x)}{q_t(x)}\right\}=1-\mathrm{TV}(p_t,q_t),
\end{equation}
Under a geometric model with common independent acceptance probability $\alpha$, the expected \emph{emitted} length, including the bonus token, is $\mathrm{MAL}(\alpha,K)=\sum_{j=0}^{K}\alpha^j$, an approximation since rejection changes later conditioning; Proposition~\ref{prop:block} gives a stage-specific bound without independence, and the reported greedy decoder accepts on exact token agreement (Proposition~\ref{prop:greedy}).
{RoPE}~\citep{su2024roformer} rotates coordinate pair $i$ of each head by $m\theta_i$ with $\theta_i=10000^{-2i/d}$, so $R(m)$ is block diagonal with $2\times2$ blocks and any block-wise property holds for the whole matrix; italic $q$ is the drafter's distribution, bold $\mathbf{q}_m$ and $\mathbf{k}_n$ a query at $m$ and a key at $n$.

\section{Related Work}
\label{sec:related}

\citet{sadhukhan2025magicdec} show speculative decoding remains profitable at large batch for long sequences using drafters with sparse {KV} caches, including a {StreamingLLM}-style windowed cache, but their draft budgets at contexts to 100k sit inside the models' native rotary range, so the position problem never arises; we run at batch~1 and one to two orders of magnitude longer.
A parallel line attacks the verification side rather than the drafter: {SpecPV}~\citep{tan2025specpv} verifies against a sparse {KV} cache, accepting divergence from full-context generation, and repairs its {EAGLE-3} drafter's positional range by {YaRN} fine-tuning, the training-based response we avoid; {TriForce}~\citep{sun2024triforce} and {TokenSwift}~\citep{wu2025tokenswift} are adjacent.  We leave verification exact and act only on the drafter's view, so the two directions compose rather than compete.
\citet{xiao2024streamingllm} enable unbounded streaming by pairing a rolling window with pinned sinks, assigning positions \emph{within the cache}; the load-bearing distinction is whether a key's index is stable for its lifetime, theirs being cache-relative and ours text-relative up to a constant per band, so their visible cache is re-rotated every decode step where ours is rotated once at entry, a recurring bandwidth cost against a one-off one.
\citet{yang2026longspec} and \citet{wang2026specla} solve the same position-generalization problem by \emph{training} long-context drafters, and Windowed-{MTP}~\citep{valliappan2026windowed} is training-free but windows a multi-token-prediction head inside the target checkpoint, sharing the target's positional regime; we window an external drafter, which is what forces the position problem.
Retrieval- and $N$-gram drafters~\citep{saxena2023pld,fu2024lookahead,ou2024anpd,he2024rest} need an exact prior occurrence, and retrieval-\emph{augmented} speculative decoding~\citep{chen2025rapid} instead hands the drafter a sparse retrieved subset of the prefix, which bounds its memory without bounding its index range, so the position problem of \secref{sec:position} persists there and re-basing applies unchanged; benchmarks remain short-context~\citep{xia2024specbench,abramovich2026speedbench}.
\textbf{Why the comparison here is internal.}  No external system is run as a baseline, because none of the above is defined at this operating point once the drafter's weights are fixed: the trained drafters need the training whose necessity is the question, the native-head method needs a head an external drafter lacks, the verification-side and retrieval-augmented methods depart from the full-context target's prediction rule, and $N$-gram drafters are not target-conditioned at all.  The nearest published configurations are therefore not absent from our tables but present as arms, a released drafter run as shipped being the unwindowed arm and a bounded draft cache carrying true indices the windowed-only arm, so each contrast moves one binary on one build (Appendices~\ref{app:related} and~\ref{app:baselines}, with the one alternative position mechanism we do not run).

\section{Method}
\label{sec:method}

\subsection{Setup and windowing}

Let the target be an arbitrary long-context model and the drafter $q$ a frozen short-context checkpoint, trained at length at most $L_{\mathrm{train}}$ and with rotary implementation reach $\Pmax$, required to propose $K$-token blocks at context $N\gg L_{\mathrm{train}}$.  No weight is modified anywhere, and Figure~\ref{fig:overview} gives the arrangement.
We state the method for \emph{target-conditioned} drafters, whose per-position input includes target-supplied hidden states, covering the {EAGLE} family and {DSpark}; we evaluate {EAGLE-3}~\citep{li2025eagle3} and two {DSpark} variants~\citep{inferact2026dspark,radixark2026dspark}, and a standalone drafter with no feature channel is outside what we test.
The drafter's visible attention cache holds only the most recent $W$ tokens, optionally plus $s$ sinks ($s=0$ throughout, justified in \secref{sec:ablations}), bounding its attention cost to a constant in $N$ while the target verifies every proposal against its complete context.  The drafter never runs its own attention over the prompt: at prefill its cache is written in one shot as a linear re-projection of the target's hidden states, normed and rotated at the assigned positions, so the re-basing of \secref{sec:rebasing} is applied at that write as well as on the query side during drafting.  Windowing therefore bounds how many positions the drafter attends over, not how much each position knows, so a window of features reaches further than a window of tokens.  We choose $W\lesssim L_{\mathrm{train}}$ to restore the visible attention length to the drafter's training range, which alone guarantees no agreement with a truncated-target pass; \secref{sec:gaps} measures the deployed fidelity.

\subsection{The position problem}
\label{sec:position}

With no distant sinks, windowing bounds visible \emph{offsets}, every $n-m$ satisfying $|n-m|\le W$.  What it does \textbf{not} bound are the \emph{indices}: at context length $N$ the drafter still evaluates $R(m)$ for $m\sim N$, and once $m\ge\Pmax$ the rotary transform is evaluated beyond the precomputed table or past where argument reduction $m\theta_i\bmod2\pi$ remains accurate.  The mathematical definition of {RoPE} is untouched; what fails is a particular implementation.  The failure is quiet, degrading acceptance and therefore speedup while a correct verifier leaves the target's intended output law unchanged, and it is invisible below $\Pmax$, which is why work at moderate context has not had to address it.  The failure has been reported independently, \citet{tan2025specpv} finding released {EAGLE-3} heads ineffective at long context because the decoder layer's {RoPE} does not extrapolate.
The naive remedy, renumbering the window to $0,\dots,W-1$ each step, fails structurally: a key is rotated once on entry and reused unrotated, so a sliding window changes every cached key's local index and forces a full re-rotation, the cost {StreamingLLM} pays.

\subsection{Rotary index re-basing}
\label{sec:rebasing}

An \emph{index band} is a block of $\Pmax$ consecutive position indices, $e=\lfloor m/\Pmax\rfloor$.
Additivity, $R(a)R(b)=R(a+b)$, and orthogonality, $R(m)^\top=R(-m)$, hold block-wise and hence for the full block-diagonal $R$ (Appendix~\ref{app:ropeproof}), so the attention score between a query at $m$ and a key at $n$ is
\begin{equation}
\label{eq:rope}
\langle R(m)\mathbf{q}_m, R(n)\mathbf{k}_n\rangle = \mathbf{q}_m^\top R(m)^\top R(n)\mathbf{k}_n \overset{\mathrm{orth.}}{=} \mathbf{q}_m^\top R(-m)R(n)\mathbf{k}_n \overset{\mathrm{add.}}{=} \mathbf{q}_m^\top R(n-m)\mathbf{k}_n,
\end{equation}
depending on $n-m$ \emph{alone}, so shifting both indices by a common constant changes nothing.
Re-basing is applied at \emph{index-assignment} time, before rotation: with $\tilde m=m\bmod\Pmax$ and $m=e\Pmax+\tilde m$, every index in a band is shifted by the same $-e\Pmax$ and $\tilde n-\tilde m=n-m$ there, so by~\eqref{eq:rope} the score is exactly unchanged while every index on which $R$ is evaluated lies in $[0,\Pmax)$ for arbitrarily large true position; each key enters already assigned $\tilde m$, is rotated once, and is never re-rotated.  Below $\Pmax$ the map is the identity, so the scheme is backward compatible, and it requires $W\le\Pmax$, since a visible cache wider than a band would hold distinct keys aliasing to one assigned index; $W\lesssim L_{\mathrm{train}}\le\Pmax$ in every configuration we run.  Its boundary scope is limited, since direct scores can differ when visible keys straddle a band boundary and downstream effects can persist through cached ancestry.  The distortion there is not a small perturbation but a \emph{positive} perceived offset of order $\Pmax$, a configuration causal masking never produces in training, so the drafter evaluates rotations outside the regime it saw; Appendix~\ref{app:ropeproof} derives this and bounds the positions not certified by exactness, conservatively and not as an observed discrepancy rate.

\begin{proposition}[Common-shift position invariance]
\label{prop:invariance}
Assume the drafter's per-position inputs are identical across the two executions, which admits externally supplied features such as target hidden states, together with position-independent tokenwise operations, deterministic exact-arithmetic inference, attention whose only position dependence is fixed-frequency {RoPE} and a translation-invariant mask, and the same token computation graph in both executions.
If every position in that graph is shifted by one common constant and both executions evaluate the same mathematical rotations, the output logits agree.
\end{proposition}

Layer induction gives this from~\eqref{eq:rope} and translation invariance of the mask.  If a prediction's dependency graph lies in one band, all its positions receive the same shift $-e\Pmax$ and re-basing agrees with ideal true-index {RoPE} (Corollary~\ref{cor:rebasing}); predictions whose visible keys or cached ancestors cross a boundary are not certified.

\section{Why This Works}
\label{sec:theory}

Windowing limits the drafter's computation while verification retains the full-context target's prediction rule, and we separate two questions: when re-basing preserves a computation, and how truncation and drafter fidelity control acceptance.  The acceptance results are distributional inequalities, not assumptions that language is local.

\label{sec:strand-one}
A cache of $W$ keys does not make a multilayer drafter a function of the last $W$ tokens, and for a target-conditioned drafter it is not meant to, so two dependencies separate.  \emph{Index} dependency, which positions' rotary indices enter the computation, stays local, bounded by $M=1+\sum_{\ell}r_\ell$ under purely local drafter attention (Proposition~\ref{prop:span}), and is the only notion the boundary bound of Appendix~\ref{app:ropeproof} uses; \emph{information} dependency, which tokens influence the output, spans the whole prefix by construction and measures 0.22--0.33 in total variation (\secref{sec:e2delta}).  The bounds below need neither, using the proposal actually produced.

\subsection{The truncation--fidelity decomposition}
\label{sec:decomposition}

Write $x_{<t}=(\mathbf d_t,\mathbf w_t)$ with $\mathbf w_t$ the last $\min\{W,t-1\}$ tokens and $p_t(\cdot):=p(\cdot\mid\mathbf d_t,\mathbf w_t)$; let $\ptrunct$ be the target distribution from a fresh pass on $\mathbf w_t$ alone under fixed mask, input, and position conventions, not a conditional average over distant prefixes; and let $q_t$ be the drafter's \emph{actual deployed} proposal, including its persistent cache, eviction, re-basing, and rollback state.  The triangle inequality gives
\begin{equation}
\label{eq:decomp}
\mathrm{TV}(p_t,q_t)
\le
\underbrace{\mathrm{TV}(p_t,\ptrunct)}_{\text{(A) truncation gap}}
+
\underbrace{\mathrm{TV}(\ptrunct,q_t)}_{\text{(B) drafter fidelity gap}},
\end{equation}
an upper bound rather than an additive identity; with~\eqref{eq:alpha} and the reverse triangle inequality it also gives $|\alpha_t-(1-\mathrm{TV}(p_t,\ptrunct))|\le\mathrm{TV}(\ptrunct,q_t)$.
\label{sec:decomp-buys}
Term (A) is measurable without a drafter, for a fixed target, comparator, and law, while term (B) measures agreement between that comparator and the implemented drafter; choosing $W\lesssim L_{\mathrm{train}}$ bounds neither, which is why we measure fidelity rather than assume it small.
Since the inequality holds for any intermediate distribution, the comparator is a modeling choice governed by one principle, that \emph{$\ptrunct$ is the target restricted to the drafter's information access}.  Fresh truncation is the instance matching a standalone drafter and the one we measure; the instance a target-conditioned drafter faces keeps the target's full-context keys and values at verified positions, so fresh truncation reads (A) conservatively (Appendix~\ref{app:extended}).

\subsection{The bound: sampling and greedy verification}
\label{sec:bound}

Let $\mu$ be any law on prediction states, including positions, prefixes, and the implementation state that determines the proposal, with all quantities in a bound using that same law, and define $\bar A:=\mathbb E_\mu[\mathrm{TV}(p_t,\ptrunct)]$ and $\delta:=\mathbb E_\mu[\mathrm{TV}(\ptrunct,q_t)]$; $\delta$ is a definition, not a small-error assumption, and neither stationarity nor independence is required.

\begin{proposition}[Sharp mean acceptance bounds]
\label{prop:bound}
For standard speculative rejection sampling and every $\mu$,
$[1-\bar A-\delta]_+\le\mathbb E_\mu[\alpha_t]\le1-|\bar A-\delta|$ with $[z]_+:=\max\{z,0\}$, and in particular $\left|\mathbb E_\mu[\alpha_t]-(1-\bar A)\right|\le\delta$.
For $|\mathcal V|\ge3$, both endpoints are attainable from the two means alone.
\end{proposition}

Both follow pointwise from the triangle inequality and its reverse form, clipped and averaged; Appendix~\ref{app:sharpness} attains each endpoint with three-token laws, and Appendix~\ref{app:joint} gives strictly stronger joint forms and a threshold variant (Corollary~\ref{cor:threshold}) that our measurements show is always the weaker screen.
The reported decoder uses greedy verification, for which softmax {TV} is a diagnostic rather than observed acceptance: applied to point masses at the tokens chosen by the full target, the truncated target, and the deployed drafter, the same argument bounds greedy agreement by $[1-\rho_{\mathrm g}-\delta_{\mathrm g}]_+$ and $1-|\rho_{\mathrm g}-\delta_{\mathrm g}|$ for flip rates defined by disagreement between successive pairs (Propositions~\ref{prop:greedy} and~\ref{prop:margin}, Appendix~\ref{app:greedy}).  The lower bound makes the key point explicit: the two error sources can cancel as well as add, so a large truncation gap does not by itself force rejection.

\subsection{Context length and window selection}
\label{sec:sizing}
\label{sec:failures}

Different context lengths are governed by different laws $\mu_N$, and Proposition~\ref{prop:length} bounds the drift between two lengths by $|\bar A_N-\bar A_{N'}|+\delta_N+\delta_{N'}$, so constant mean gaps do not imply constant acceptance and choosing $W$ needs both gaps and decoding cost on a calibration grid.  Per-position bounds outside a sparse exceptional set carry over to prefix averages, and the bounds lift to whole rounds without assuming independent acceptances (Propositions~\ref{prop:finite} and~\ref{prop:block}, Appendices~\ref{app:finite} and~\ref{app:blockproof}).  The claim is therefore conditional and testable, and Appendix~\ref{app:extended} lists the correctness conditions and the five measurements that would explain an implementation rather than bound it.

\section{Experiments}
\label{sec:experiments}

No weight is updated anywhere in any run.  A {Kimi-K3} target~\citep{kimiteam2026k3} (2.8T parameters, 104B active, 93 layers, tensor parallel~8) verifies against its complete prefix while a released, frozen short-context drafter proposes $K=3$ tokens through a bounded window with re-based rotary indices; {Qwen2.5-7B-Instruct-1M}~\citep{qwen2025qwen1m} with its {EAGLE-3} head~\citep{li2025eagle3} carries the ablations and the only direct gap measurements, and {MiniMax-M3}~\citep{lai2026msa} with {EAGLE-3.1}~\citep{eldenk2026attndrift} a third ladder.  Two axes of generality are deliberate, three target--drafter pairs and five corpora (Appendix~\ref{app:datasets}), so the result can be a property of neither {Kimi-K3}, {DSpark}, nor {PG19}~\citep{rae2020pg19} prose.  Both mechanisms are a patch to {vLLM}~\citep{kwon2023vllm}, and every measurement is taken in that engine.
Three evaluation laws appear below and \secref{sec:theory} forbids substituting a rate measured under one into a bound stated under another: \emph{deployment}, reached proposal states in an actual run; \emph{corpus-position}, teacher-forced positions on a stride from held-out documents; and \emph{fresh-drafter}, corpus-position with the drafter run as a clean forward pass.  Timing and acceptance use the first, the gaps of \secref{sec:gaps} the third.

\subsection{Setup and measurement controls}
\label{sec:setup}

Every run is batch size one, greedy, with decode-phase timing from the engine's per-phase counters so prefill is excluded, and speedup computed per sample against a paired target-only run on the identical prompt; throughout, {MAL} is \emph{emitted} tokens per round, accepted drafts plus one.
Because each comparison is a ratio between two separately booted engine instances, anything taxing one arm more manufactures or destroys the effect; four controls are load-bearing (Appendix~\ref{app:controls}), the sharpest being symmetric compiled-graph capture, one arm compiled against the other eager moving a ratio from $0.38\times$ to $2.21\times$.

\subsection{Windowed drafting to two million tokens}
\label{sec:e4}

\begin{table}[t]
\caption{Three arms on {Kimi-K3}, {PG19}, $K=3$, $W=4096$, $n=5$ per cell.  $S$ is decode-phase speedup over a paired target-only run; {MAL} is emitted tokens per round out of $K+1=4$.  Only the drafter's view of context and position varies, and bold marks our arm rather than the winning cell.  Sample standard deviations of $S$ are $0.26$--$0.66$ for profitable cells and $0.02$--$0.10$ for collapsed ones, so individual cells should not be read finely: another arm is nominally ahead in five of them, {Inferact} at 16k and 32k and {RadixArk} at 8k, 128k, and 1M, each by less than that spread, and we read none of the five as a finding.  The claim is the retention row, each arm's 2M speedup as a fraction of its own speedup at 8k.}
\label{tab:e4}
\begin{center}
\scriptsize
\begin{tabular}{rrrrrrr}
\toprule
& \multicolumn{2}{c}{unwindowed} & \multicolumn{2}{c}{windowed only} & \multicolumn{2}{c}{windowed ${}+{}$ re-based (ours)}\\
\cmidrule(lr){2-3}\cmidrule(lr){4-5}\cmidrule(lr){6-7}
$N$ & $S$ & {MAL} & $S$ & {MAL} & $S$ & {MAL}\\
\midrule
\multicolumn{7}{l}{\emph{{Inferact}}, five {MLA} draft layers, trained at 16{,}384}\\
8k   & 2.03 & 2.63 & 2.20 & 2.87 & \textbf{2.51} & 3.29\\
16k  & 2.39 & 3.12 & 2.33 & 3.04 & \textbf{2.36} & 3.08\\
32k  & 2.29 & 3.01 & 1.90 & 2.48 & \textbf{2.10} & 2.73\\
128k & 1.29 & 1.74 & 2.31 & 3.02 & \textbf{2.41} & 3.15\\
512k & 0.79 & 1.17 & 2.60 & 3.37 & \textbf{2.46} & 3.18\\
1M   & 0.66 & 1.10 & 2.24 & 2.88 & \textbf{2.55} & 3.28\\
2M   & 0.56 & 1.07 & 0.88 & 1.13 & \textbf{2.25} & 2.89\\
\midrule
\multicolumn{7}{l}{\emph{{RadixArk}}, five {GQA} draft layers, trained at 4096}\\
8k   & 1.79 & 2.30 & 1.97 & 2.50 & \textbf{1.75} & 2.24\\
16k  & 1.57 & 2.10 & 1.79 & 2.28 & \textbf{2.03} & 2.59\\
32k  & 1.63 & 2.33 & 2.05 & 2.62 & \textbf{2.09} & 2.66\\
128k & 1.08 & 2.14 & 1.90 & 2.42 & \textbf{1.63} & 2.08\\
512k & 0.55 & 2.18 & 1.95 & 2.46 & \textbf{1.91} & 2.42\\
1M   & 0.32 & 1.97 & 2.01 & 2.54 & \textbf{1.63} & 2.06\\
2M   & 0.14 & 1.35 & 0.87 & 1.11 & \textbf{2.08} & 2.63\\
\midrule
\multicolumn{7}{l}{\emph{retained at 2M as a fraction of the same arm at 8k}}\\
{Inferact} & \multicolumn{2}{r}{$0.28$} & \multicolumn{2}{r}{$0.40$} & \multicolumn{2}{r}{$\mathbf{0.90}$}\\
{RadixArk} & \multicolumn{2}{r}{$0.08$} & \multicolumn{2}{r}{$0.44$} & \multicolumn{2}{r}{$\mathbf{1.19}$}\\
\bottomrule
\end{tabular}
\end{center}
\end{table}

The controlled comparison holds the target, prompts, block size $K$, and sample set fixed, varying only how the drafter sees context and position.  Two released {DSpark} drafters run against the same {Kimi-K3} target and differ in draft-layer attention, {Inferact}~\citep{inferact2026dspark} with five {MLA}~\citep{deepseek2024v2} layers trained at 16{,}384 and {RadixArk}~\citep{radixark2026dspark} with five {GQA}~\citep{ainslie2023gqa} layers trained at 4096; both read target hidden states from five layers of 93 and neither runs without the target's activations.  Workload is {PG19}, {EOS} ignored, 4096 generated tokens, seven tiers from 8k to 2M, $n=5$ per cell, $W=4096$, $s=0$, one {MI355X} node.
Three arms per drafter separate the two mechanisms of \secref{sec:method}: \emph{unwindowed}, the drafter attending over the whole prefix; \emph{windowed only}, its visible cache bounded at $W$ but its rotary indices left as the target supplies them; and \emph{windowed ${}+{}$ re-based}, our method, which additionally maps indices into $[0,\Pmax)$.  All three run on one engine build, and the last two differ by a single environment flag, so binary, hardware, prompts, and drafter are identical across them.

Table~\ref{tab:e4} gives the result, plotted as Figure~\ref{fig:e4}, and the claim is about flatness rather than any particular multiple: we trained neither drafter and claim nothing about drafter quality, only that whatever a released drafter is worth at short context it stays worth roughly that at $250\times$ the length.
Windowed and re-based, {Inferact} retains 90\% of its 8k speedup at 2M ($2.51\times$ to $2.25\times$) and {RadixArk} retains 119\% ($1.75\times$ to $2.08\times$).  Neither ratio is a trend in length, and the one above unity least of all: {RadixArk}'s 8k denominator is the lowest of its three arms at that tier, so a retention above one reflects a weak reference cell rather than a gain from context, and what the two numbers support is the flat claim only, that the value at 2M is the value at 8k to within the ladder's own cell-to-cell spread.  {Inferact} sits above {RadixArk} at every tier, by $0.01$ to $0.92\times$, the level belonging to the drafter while the flatness belongs to the mechanism.

Both other arms lose that flatness, for different reasons, which is the table's most diagnostic feature.  Unwindowed drafting decays continuously, keeping 28\% and 8\% of its 8k speedup at 2M and crossing break-even between 128k and 512k and its primary failure is cost rather than acceptance, unwindowed {RadixArk} still reaching {MAL} 2.18 at 512k and 1.97 at 1M while running at $0.55\times$ and $0.32\times$, each proposal having merely become too expensive to be worth making.  The ratio $S/\mathrm{MAL}$ isolates this, sitting at $0.76$--$0.79$ in every windowed cell, the constant of a round dominated by one full-context verification, and falling to 0.52 and 0.10 unwindowed at 2M.  {Inferact} unwindowed pays twice, shedding acceptance as well from 128k ({MAL} 3.01 at 32k against 1.74), as one expects of a drafter trained at 16{,}384.

Windowing alone removes the cost failure, and below $\Pmax$ it \emph{is} our method: the two windowed arms differ by at most $0.38\times$ at every tier through 1M, which does not exceed the $0.26$--$0.66$ per-cell standard deviation, and the sign is inconsistent, so we read them as one arm measured twice, which is what the theory requires of a map provably the identity there (Corollary~\ref{cor:rebasing}, quantified in Appendix~\ref{app:rebase}).
Past $\Pmax$ it collapses.  Both drafters declare a maximum position of exactly $\Pmax=1{,}048{,}576$, reached by {YaRN}~\citep{peng2024yarn} scaling of a shorter native range, so 1M is the last tier inside their rotary reach and 2M the first outside it, while the {NoPE}~\citep{kazemnejad2023nope} target serves past that point and still hands them true absolute positions (Appendix~\ref{app:rebase}).  At 2M, windowing alone gives {MAL} 1.13 and 1.11, so only the bonus token is emitted, at $0.88\times$ and $0.87\times$, with $S/\mathrm{MAL}$ still 0.78, so drafting cost is untouched and this is purely an acceptance failure, and the per-cell spreads quoted with Table~\ref{tab:e4} put both failures beyond sampling noise.  The two mechanisms are therefore not substitutes, and the shape is a cliff at $\Pmax$ rather than a decay, which is what \secref{sec:position} predicts for an out-of-range rotary access.

\subsection{Beyond {Kimi-K3} and {DSpark}}
\label{sec:e1}

Everything above uses one target and one drafter family, so the standing objection is that flatness is a property of {Kimi-K3} or of how {DSpark} drafters are built.  Two further pairs answer it, each differing from the main ladder in a way that could have broken the method.
The same comparison on {Qwen2.5-7B-Instruct-1M} with its {EAGLE-3} head, at $W=2048$ and $K=3$ over ten {PG19} context lengths, reproduces the pattern at smaller scale (Table~\ref{tab:e1} and its qualifications, Appendix~\ref{app:tables}): windowed {MAL} is flat at 1.61--1.81 from 2k to 512k with speedup $1.31$--$1.83\times$, against {MAL} 1.23--1.50 and $1.11$--$1.22\times$ unwindowed, which pays the target's full context every draft step.  Read as \secref{sec:e4} reads it, the windowed arm retains $1.22$ of its 2k speedup at 512k against $0.91$ unwindowed, on a far smaller target and a different drafter family.  We stop at 512k because this pair's 1M target denominator used a different checkpoint variant.

A third pair, {MiniMax-M3} with a released {EAGLE-3.1} head at $W=4096$ and $K=3$, is the sharper test: this target's attention is block-sparse and it loses only 26\% of its decode rate over $250\times$ of context, so a windowed drafter cannot win by outliving a collapsing baseline.  It wins anyway (Table~\ref{tab:m3}, Appendix~\ref{app:more}): over nine tiers from 4k to 1M the windowed arm spans $2.08$--$2.62\times$ and retains $1.11$ of its 4k speedup at 1M, while the unwindowed arm falls from $2.38\times$ to $1.75\times$ and retains $0.74$, the difference at 1M carried by acceptance, {MAL} 3.32 against 2.49.  Flatness therefore holds on all three pairs, here against a target that does not itself collapse, which is the case the claim has to survive; this pair isolates windowing alone, its drafter's rotary reach exceeding the deepest position served.
Speed is also not bought with retrieval, the windowed arm answering 120 of 120 {RULER}~\citep{hsieh2024ruler} single-needle tasks from 8k to 1M exactly as the target does, where the answer hangs on one short span the window excludes (Appendix~\ref{app:niah}).

\subsection{Window and sink ablations}
\label{sec:ablations}

Varying $W$ with all else fixed, windowing's benefit appears at 4k and grows monotonically to {MAL} 1.67 against 1.26 unwindowed at 32k, the crossover of Table~\ref{tab:e4} at another scale, while mean {MAL} at $W=1024$ and $W=2048$ differs by at most 0.02, so returns are saturated by 1024 on this drafter without establishing equivalence (Table~\ref{tab:e6}, Appendix~\ref{app:tables}).  Both findings recur at 1M on {Kimi-K3}: acceptance is flat across a $4\times$ window range (Appendix~\ref{app:winabl}), and a sink sweep at that length justifies $s=0$ on {Inferact}, though not on {RadixArk}, which responds in the opposite direction (Appendix~\ref{app:sinkabl}).

\subsection{Measuring the two gaps}
\label{sec:gaps}
\label{sec:e2pilot}
\label{sec:e2delta}

Term (A) needs no drafter: on {Qwen2.5-7B-Instruct} at 32k we compare the target's full-prefix distribution $p_t$ against a fresh pass on the last $W$ tokens only, over a pre-registered grid of 34 windows at 1750 positions in 14 {SPEED-Bench} documents (Appendix~\ref{app:grid}).  The mean gap $\bar A$ falls from 0.2195 at $W=128$ to 0.0444 at $W=16384$, the greedy flip rate from 0.247 to 0.063, and at the deployed $W=2048$ truncation changes the target's greedy token at 12.1\% of positions.
For term (B) we reimplemented the {EAGLE-3} drafter so its output could be compared against that comparator at matched positions and windows, under the \textbf{fresh-drafter law}: it runs as a clean forward pass, not from a deployed cache, so these numbers characterize the checkpoint rather than the shipped system.  Table~\ref{tab:gaps} in Appendix~\ref{app:gapdetail} reports both gaps at matched positions, where the bound of Proposition~\ref{prop:bound} also proves informative but loose, and says two further things.

\textbf{Fidelity is large and flat}, $\delta\approx0.68$ across a $64\times$ window range, so acceptance here is drafter-limited rather than window-limited: the truncation gap falling by 0.159 over the sweep moves measured acceptance by only 0.019, the drafter being already wrong in the directions truncation would hurt, so the two errors overlap rather than add.  This is why we refuse to call term (B) small, and it is consistent with the deployment-law saturation at $W\ge1024$ in \secref{sec:ablations} without establishing its cause.
\textbf{The feature channel is large}, 0.22--0.33 in {TV} between a drafter fed full-pass and one fed truncated-pass target states, so the drafter is not a function of its window, which is how a 1024-token head serves a 2M-token target.
\section{Limitations}
\label{sec:limitations}

The ladder results are single-stream, greedy, on one platform, with five samples per cell; batched serving is spot-checked on a single pair, where the ordering of \secref{sec:e4} holds at long context but not at short (Appendix~\ref{app:batch}).
Re-basing is exercised only on the pair whose drafters have a rotary reach the served context crosses, so the cliff rests on that pair, and because the speedups and the gap measurements come from \emph{different} model pairs we do not claim the gaps explain the speedups.
Profitability is conditional rather than universal, requiring the speculative cost per emitted token to stay below the matched target-only cost, which expensive verification can prevent (Appendix~\ref{app:more}).
Each gap is measured under one evaluation law and one comparator, which \secref{sec:bound} forbids substituting into one another's bounds, and neither is monotone in $W$ by construction.
Output equivalence is only partly measured and we make no formal lossless-decoding claim, the residual differences we do observe tracking engine nondeterminism rather than speculation; neither mechanism touches the verifier, so whatever output-equivalence property the implementation has is inherited and every committed token is still chosen by the target's own full-context verification.
The theory assumes fixed-frequency {RoPE} with translation-invariant masking, leaving recurrence, learned absolute positions, and finite precision for separate treatment; Appendix~\ref{app:remaining} details the outstanding controls and the theory's scope.

\subsection*{AI use statement}

In this work, we used generative {AI} tools to proof-read the text of the paper and to assist with writing it.
We have reviewed all {AI}-assisted text, citations, and mathematical transcription.
We take responsibility for the final content of this work, including text, claims or artifacts produced with the aid of generative {AI}.

\subsection*{Reproducibility statement}

\Secref{sec:method} specifies attention masking and rotary index re-basing and requires no inference-time weight updates; \secref{sec:theory} fixes the evaluation law, comparator, proposal, bounds, and emitted-length convention, with Appendix~\ref{app:ropeproof} giving the architectural scope of re-basing exactness; and \secref{sec:experiments} names the evaluation law for every reported quantity, lists the four measurement controls, and records the pinned engine build under which the length ladder was revalidated.
Model checkpoints, prompt manifests, per-sample logs, the truncation and fidelity sweeps, and the {vLLM} patch implementing both mechanisms will accompany the paper.

\bibliography{arxiv}
\bibliographystyle{arxiv}

\appendix
\section{Proofs and Supplementary Guarantees}

\subsection{{RoPE}, cached dependency spans, and band boundaries}
\label{app:ropeproof}

\textbf{Rotation identities.}
For coordinate pair $i$, fixed-frequency {RoPE} uses
\[
R_i(m)=
\begin{pmatrix}
\cos(m\theta_i)&-\sin(m\theta_i)\\
\sin(m\theta_i)&\cos(m\theta_i)
\end{pmatrix}.
\]
The angle-addition identities give $R_i(a)R_i(b)=R_i(a+b)$, while transposition changes the sine signs and gives $R_i(m)^\top=R_i(-m)$.
Since $R(m)$ is block diagonal, both identities hold for $R$ and imply \eqref{eq:rope}.

\textbf{From scores to outputs.}
Under Proposition~\ref{prop:invariance}'s assumptions, shift all positions in the complete token computation graph by one constant.
Layer-zero states agree.
If lower-layer states agree, their unrotated queries, keys, and values agree; relative offsets and visible-key sets also agree, so all attention scores, weights, and value sums agree.
Tokenwise operations and residual connections preserve equality, completing the layer induction.
A {KV} cache stores earlier states, so the same hypotheses must hold for the computations that created those states, not only for the current query.
Externally supplied features are permitted because the target computes them once, independently of the drafter's index assignment, so they are identical in both executions; the proof does not apply unchanged to learned absolute embeddings, features that themselves depend on the drafter's index map, context-dependent rotary frequencies, or finite-precision equality.

\begin{corollary}[Within-band re-basing]
\label{cor:rebasing}
Compare modulo re-basing with ideal mathematical {RoPE} at true indices under Proposition~\ref{prop:invariance}'s hypotheses.
If a prediction's entire dependency graph lies within one index band, all its positions receive the same shift $-e\Pmax$, so the outputs agree.
This does not certify predictions whose visible keys or cached ancestors cross a boundary.
\end{corollary}

\begin{proposition}[Local dependency span]
\label{prop:span}
Consider an $L$-layer causal Transformer with token-only input embeddings and tokenwise non-attention operations.
At layer $\ell$, attention at token position $n$ reads lower-layer states only at positions $n-r_\ell,\ldots,n$, restricted to the available prefix.
Assume the same rules during prefill and decoding, with no recurrence, global attention, pinned distant token, or external state.
Put
\[
R:=\sum_{\ell=1}^{L}r_\ell,\qquad M:=R+1.
\]
The state at token $n$ depends only on $x_{\max(1,n-R):n}$.
A prediction for token $t$ read from token $t-1$ therefore depends on at most the last $M$ tokens, including with a persistent {KV} cache.
\end{proposition}

\textbf{Two dependencies.}
For a target-conditioned drafter the token-only hypothesis fails and the information dependency is the whole prefix by design.
What the boundary bound~\eqref{eq:boundary} needs is the weaker \emph{index} dependency, the set of positions whose rotary indices enter the computation, and that remains bounded by $M$ whenever the drafter's own attention is local, because the target supplies a feature at a position without contributing any index of the drafter's assignment.
Proposition~\ref{prop:span} should be read with $M$ in that sense; its stronger token-level reading additionally requires token-only inputs.

\begin{proof}
Layer zero at $n$ depends only on $x_n$.
If a layer-$(\ell-1)$ state at $j$ depends only on
$x_{\max(1,j-\sum_{k<\ell}r_k):j}$, the union over
$j\in[n-r_\ell,n]$ lies inside
$[n-\sum_{k\le\ell}r_k,n]$.
Induction proves the claim; caching these same states does not alter their dependencies.
\end{proof}

The implementation's window convention determines $r_\ell$.
If a layer sees itself and $W-1$ predecessors, then $r_\ell=W-1$ and $M=1+L(W-1)$.
A global prefill is not covered merely because later decoding uses a local cache.

\textbf{A sufficient exceptional set.}
Compare modulo re-basing with ideal mathematical {RoPE} at true indices using the same mask, cache rules, and token sequence.
If every token in a prediction's dependency graph lies in one band, Proposition~\ref{prop:invariance} gives equal outputs.
Let $S$ be the set of predictions whose $M$-token dependency interval crosses a band boundary.
For any $N$ consecutive prediction positions, its covered fraction obeys
\begin{equation}
\label{eq:boundary}
\beta_N\le\min\left\{1,\frac{M}{\Pmax}+\frac{2M}{N}\right\}.
\end{equation}
For $M<\Pmax$, at most $N/\Pmax+2$ boundary intervals can intersect a run of length $N$, and each marks at most $M$ predictions; for $M\ge\Pmax$ the clipped bound is one.
This is a conservative bound on positions not certified by exactness, not on discrepancy magnitude or observed errors.
For example, $L=32$, $W\approx1$k, and $\Pmax\approx128$k give a long-run term near 25\%, while the drafters of \secref{sec:e4}, with $L=5$, $W=4096$, and $\Pmax=1{,}048{,}576$, give $M=20{,}476$ and a long-run term near 2\%.
A direct-edge fraction near $W/\Pmax$ (about 0.8\%) does not bound downstream output effects.

\subsection{Sharpness of the sampling mean bounds}
\label{app:sharpness}

Fix arbitrary $a,b\in[0,1]$ and use three vocabulary elements.
If $a+b\le1$, set
\[
p_t=(1,0,0),\quad
\ptrunct=(1-a,a,0),\quad
q_t=(1-a-b,a+b,0).
\]
The adjacent {TV} gaps are $a,b$ and $\alpha_t=1-a-b$.
If $a+b\ge1$, instead set
\[
p_t=(1,0,0),\quad q_t=(0,1,0),\quad
\ptrunct=(1-a,1-b,a+b-1),
\]
which has the same adjacent gaps and $\alpha_t=0$.
Thus the lower endpoint of Proposition~\ref{prop:bound} is always attained.
For the upper endpoint use
\[
\ptrunct=(1,0,0),\quad p_t=(1-a,a,0),\quad q_t=(1-b,b,0),
\]
for which $\mathrm{TV}(p_t,q_t)=|a-b|$.
This proves sharpness over arbitrary distributions and state laws; it does not assert that one trained model realizes every construction.

\subsection{Joint forms and acceptance drift}
\label{app:joint}

Write $A_t:=\mathrm{TV}(p_t,\ptrunct)$ and $B_t:=\mathrm{TV}(\ptrunct,q_t)$.
Retaining the joint law of the two gaps rather than only their means gives
\[
\mathbb E_\mu\bigl[(1-A_t-B_t)_+\bigr]
\le \mathbb E_\mu[\alpha_t]
\le1-\mathbb E_\mu\bigl[|A_t-B_t|\bigr],
\]
which implies Proposition~\ref{prop:bound} by convexity of $[\,\cdot\,]_+$ and of the absolute value, and is strictly stronger whenever the two gaps are not almost surely ordered the same way.

\begin{proposition}[Uniform guarantees and acceptance drift]
\label{prop:length}
If $\bar A_N\le a$ and $\delta_N\le b$ for every tested $N$, then $\inf_N\mathbb E_{\mu_N}[\alpha_t]\ge[1-a-b]_+$.
For any two lengths,
\[
\left|\mathbb E_{\mu_N}[\alpha_t]-\mathbb E_{\mu_{N'}}[\alpha_t]\right|
\le|\bar A_N-\bar A_{N'}|+\delta_N+\delta_{N'}.
\]
The same statements hold for greedy verification after substituting $\alpha^{\mathrm g}$, $\rho_{\mathrm g}(W,0)$, and $\delta_{\mathrm g}$.
\end{proposition}

\begin{proof}
Apply Proposition~\ref{prop:bound} at each length.
For the drift bound, insert $1-\bar A_N$ and $1-\bar A_{N'}$ and apply the deviation bound of Proposition~\ref{prop:bound} twice.
\end{proof}

In particular, if truncation is constant across lengths and fidelity is at most $b$, the guaranteed pairwise drift is at most $2b$.
A deployment sweep and a paired intervention that changes retained distant context at fixed prediction positions are different laws, and quantities measured under one may not be substituted into a bound stated under the other.

\subsection{Extended discussion of the decomposition}
\label{app:extended}

\textbf{What the decomposition buys.}
Calling term (B) ``fidelity'' does not identify its cause.
Choosing $W\lesssim L_{\mathrm{train}}$ bounds the visible attention window and assigning rotary indices inside $[0,\Pmax)$ avoids out-of-range rotations, but neither step bounds term (B), guarantees that mid-document windows resemble training inputs, nor proves that an unwindowed baseline must fail.
We therefore measure fidelity rather than assuming it is small.
The inequalities would hold with any specified comparison distribution, but changing the comparator changes both gaps and their interpretation; throughout, $\ptrunct$ denotes the fresh-window target pass.

\textbf{The threshold form.}
On our measurements the mean form dominates it by 0.14 to 0.27 at every window, which is why the main text reports the mean form.
For $\tau\in[0,1]$ define the events $\mathcal L_\tau:=\{\mathrm{TV}(p_t,\ptrunct)\le\tau\}$ and $\mathcal R_\tau:=\mathcal L_\tau^c$ under $\mu$, with $\rho(W,\tau):=\mu(\mathcal R_\tau)$.

\begin{corollary}[Threshold form]
\label{cor:threshold}
For every $\tau\in[0,1]$,
\[
\mathbb E_\mu[\alpha_t]
\ge[1-\bar A-\delta]_+
\ge[(1-\rho(W,\tau))(1-\tau)-\delta]_+,
\]
and $\bar A=\int_0^1\rho(W,u)\,du$.
The two unclipped lower-bound expressions are equal at a fixed $\tau$ exactly when $\mathrm{TV}(p_t,\ptrunct)\in\{\tau,1\}$ almost surely.
\end{corollary}

\begin{proof}
Pointwise, $\mathrm{TV}(p_t,\ptrunct)\le\tau\mathbf{1}_{\mathcal L_\tau}+\mathbf{1}_{\mathcal R_\tau}$; average and apply Proposition~\ref{prop:bound}.
Equality requires the nonnegative pointwise slack to vanish almost surely, giving the stated condition, and clipping can add further equalities when both bounds are zero.
The integral identity is the tail formula $z=\int_0^1\mathbf{1}\{z>u\}\,du$.
\end{proof}

\textbf{Correctness and reached stages.}
Under a correct greedy verifier, accepted tokens match the full target and the first rejection is replaced by that target's greedy token, provided cache updates, tie-breaking, and target computations are consistent.
Sampling correctness additionally requires the prescribed residual correction rule; it does not follow from an acceptance bound alone.
At reached stage $i$, let $\mu_i$ be the state law conditional on the first $i-1$ proposals being accepted, and set $a_i=[1-\bar A_i-\delta_i]_+$, or $a_i=[1-\rho_{\mathrm g,i}(W,0)-\delta_{\mathrm g,i}]_+$ for greedy verification.
These are the quantities entering Proposition~\ref{prop:block}.

\textbf{What must be measured.}
To explain a particular implementation rather than merely bound it, measure, under the corresponding evaluation law:
\begin{enumerate}
\item acceptance and both gaps across context lengths, which give Proposition~\ref{prop:length}'s drift envelope rather than exact flatness (\secref{sec:e4}, \secref{sec:e1});
\item both gaps on a fixed calibration grid in $W$, counting the five agreement patterns of Appendix~\ref{app:greedy} for greedy runs (\secref{sec:ablations}, \secref{sec:gaps});
\item the deployed cached drafter at the same reached, mid-document states as both target passes, a fresh drafter pass being a different comparison (not yet available; Appendix~\ref{app:remaining});
\item output invariance separated from direct score invariance through the full dependency graph, with boundary outputs measured before any quantitative discrepancy claim (not attempted);
\item accepted drafts, emitted tokens, and runtime reported separately, using Proposition~\ref{prop:block} quantitatively only with reached-stage gap estimates (\secref{sec:setup}).
\end{enumerate}

\subsection{Greedy case analysis, sharpness, and a margin condition}
\label{app:greedy}

\begin{proposition}[Greedy acceptance and error cancellation]
\label{prop:greedy}
Fix deterministic tie-breaking and let $x_t^p,x_t^{\mathrm{trunc}},x_t^q$ be the tokens selected by the full target, the truncated target, and the deployed drafter, with $\alpha_t^{\mathrm g}:=\mathbf{1}\{x_t^p=x_t^q\}$, $\rho_{\mathrm g}:=\mu(x_t^p\ne x_t^{\mathrm{trunc}})$, and $\delta_{\mathrm g}:=\mu(x_t^{\mathrm{trunc}}\ne x_t^q)$.
For a verifier that accepts exactly when the draft token equals the full target's greedy token,
\[
[1-\rho_{\mathrm g}-\delta_{\mathrm g}]_+
\le\mathbb E_\mu[\alpha_t^{\mathrm g}]
\le1-|\rho_{\mathrm g}-\delta_{\mathrm g}|,
\]
and exactly $\mathbb E_\mu[\alpha_t^{\mathrm g}]=\mu(x_t^p=x_t^{\mathrm{trunc}}=x_t^q)+\mu(x_t^p=x_t^q\ne x_t^{\mathrm{trunc}})$.
\end{proposition}

This is Proposition~\ref{prop:bound} applied to point masses at the three selected tokens, with the acceptance event split according to whether the common full-target/draft token equals the truncated-target token.
Write $a=x_t^p$, $b=x_t^{\mathrm{trunc}}$, and $c=x_t^q$.
The following cases are disjoint and exhaustive:
\begin{center}
\begin{tabular}{lccc}
\toprule
Token relation & Truncation flip & Fidelity flip & $\alpha_t^{\mathrm g}$\\
\midrule
$a=b=c$ & 0&0&1\\
$a\ne b=c$ & 1&0&0\\
$a=b\ne c$ & 0&1&0\\
$a=c\ne b$ & 1&1&1\\
$a,b,c$ all distinct & 1&1&0\\
\bottomrule
\end{tabular}
\end{center}
Hence exact acceptance is the all-equal probability plus the cancellation probability.
Writing $P_{00}$ for the all-equal probability and $P_{11}$ for the total probability of the final two rows gives the tighter joint interval
\[
P_{00}\le\mathbb E_\mu[\alpha_t^{\mathrm g}]\le P_{00}+P_{11}.
\]
For vocabularies of at least three symbols the mean bounds are sharp: when desired flip rates $u,v$ satisfy $u+v\le1$, place masses $1-u-v,u,v$ on all-equal, truncation-only, and fidelity-only states.
When $u+v\ge1$, place masses $1-v,1-u,u+v-1$ on truncation-only, fidelity-only, and all-distinct states.
For the upper endpoint, if $u\ge v$, place masses $1-u,u-v,v$ on all-equal, truncation-only, and cancellation states; swap roles when $v\ge u$.
A binary vocabulary is the one exception, and the qualification matters: there $x_t^p\ne x_t^{\mathrm{trunc}}$ and $x_t^{\mathrm{trunc}}\ne x_t^q$ together force $x_t^p=x_t^q$, so the all-distinct state cannot occur and the constructions above are unavailable.  The sharp mean bounds tighten to $|1-u-v|\le\mathbb E_\mu[\alpha_t^{\mathrm g}]\le1-|u-v|$; at $u=v=1$ acceptance is exactly one, so the unrestricted lower endpoint zero is unattainable there.

\begin{proposition}[Greedy margin condition]
\label{prop:margin}
Let $p_t,\ptrunct,q_t$ now denote softmax vectors whose maxima determine greedy tokens.
Suppose $p_t$ has unique maximizer $x_t^*$ and margin
\[
m_t:=p_t(x_t^*)-\max_{x\ne x_t^*}p_t(x).
\]
If
\[
2\bigl(\mathrm{TV}(p_t,\ptrunct)+\mathrm{TV}(\ptrunct,q_t)\bigr)<m_t,
\]
then the drafter selects $x_t^*$.
\end{proposition}

\begin{proof}
If another token $y$ maximized $q_t$, then
\[
m_t\le |p_t(x_t^*)-q_t(x_t^*)|+|q_t(y)-p_t(y)|
\le2\mathrm{TV}(p_t,q_t)
\le2\bigl(\mathrm{TV}(p_t,\ptrunct)+\mathrm{TV}(\ptrunct,q_t)\bigr),
\]
a contradiction.
\end{proof}
Without a margin premise, arbitrarily small softmax {TV} can coexist with greedy disagreement.

\subsection{Finite-prefix corrections and stationarity}
\label{app:finite}

Here only, let $\mu_N$ average positions $1,\ldots,N$ uniformly in one fixed process.
At fixed $t$, put
\[
a_t:=\mathbb E[\mathrm{TV}(p_t,\ptrunct)],\qquad
b_t:=\mathbb E[\mathrm{TV}(\ptrunct,q_t)].
\]

\begin{proposition}[Finite-prefix and exceptional-position correction]
\label{prop:finite}
Let $S\subseteq\mathbb N$ be deterministic, $t_0\in\mathbb Z_{\ge1}$, $a,b\in[0,1]$, and $\eta_A,\eta_B\ge0$.
Suppose that for every $t\ge t_0$ outside $S$,
$|a_t-a|\le\eta_A$ and $|b_t-b|\le\eta_B$.
Set
\[
\beta_N:=\frac{|S\cap\{1,\ldots,N\}|}{N},
\qquad
\epsilon_N:=\frac{t_0-1}{N}+\beta_N.
\]
Then
\[
|\bar A_N-a|\le\eta_A+\epsilon_N,\qquad
|\delta_N-b|\le\eta_B+\epsilon_N.
\]
Each endpoint of Proposition~\ref{prop:bound} differs from that computed at $(a,b)$ by at most
$\eta_A+\eta_B+2\epsilon_N$.
\end{proposition}

\begin{proof}
Regular positions contribute at most $\eta_A$ each to the first average.
At most $t_0-1+N\beta_N$ other positions contribute, each by at most one.
The fidelity result is identical.
Both endpoint maps are one-Lipschitz in each argument, proving the last claim.
\end{proof}

Stationarity of a finite suffix is not full-target stationarity.
If, outside initialization and a deterministic exceptional set, $\ptrunct=f(z_t)$ and $q_t=g(z_t)$ for fixed maps $f,g$ of a common state $z_t$ whose law is stationary, the fidelity mean is constant there.
The finite-window case $z_t=x_{t-H:t-1}$ is the instance used above; a cache-windowed comparator and a target-conditioned drafter instead share a state that includes target-computed features, in which case stationarity of the text process alone is not enough.
It supplies neither small fidelity nor constant truncation gap, because the full target can use a growing distant prefix.

\subsection{Emitted-length proof and the geometric model}
\label{app:blockproof}

\begin{proposition}[Emitted block length without independence]
\label{prop:block}
Let $K\in\mathbb Z_{\ge0}$.  A round has $K$ \emph{scheduled} proposal stages and stops verification at the first rejection, emitting one correction token after a rejection or one bonus token once all $K$ proposals are accepted; {EOS} and output-budget termination do not interrupt this accounting.  Let $J$ be the number of accepted proposals, so the round emits $J+1$ tokens.
Suppose that at every reachable stage $i$ the hazard satisfies $\Pr(J\ge i\mid J\ge i-1)\ge a_i$ with $a_i\in[0,1]$.
Then, under either verification rule, $\mathbb E[J+1]\ge1+\sum_{j=1}^{K}\prod_{i=1}^{j}a_i$, and if every reachable stage has $a_i\ge a$ then $\mathbb E[J+1]\ge\sum_{j=0}^{K}a^j$.
No independence of acceptance events is assumed.
\end{proposition}

A pooled corpus gap does not control every reached-stage law, emitted length alone does not determine speedup, and the result assumes fixed $K$.
Write $s_i=\Pr(J\ge i\mid J\ge i-1)$ whenever the conditioning event has positive probability, and set later hazards and their lower bounds to zero once the reach probability vanishes; total expectation under $\mu_i$ gives $s_i\ge a_i$ at every reachable stage.
The probability chain rule yields
$\Pr(J\ge j)\ge\prod_{i=1}^{j}a_i$.
The tail-sum identity
\[
\mathbb E[J+1]=1+\sum_{j=1}^{K}\Pr(J\ge j)
\]
proves Proposition~\ref{prop:block}.
An unreachable stage has zero survival and all later lower products vanish.
Reached-stage laws may differ from corpus, emitted-token, or round-start laws; logits computed after the first rejection do not represent reached stages.

Under the separate geometric approximation that, conditional on a round-start state, within-round acceptances are independent with common probability $\alpha_t$,
\[
\mathrm{MAL}(\alpha_t,K)=\sum_{j=0}^{K}\alpha_t^j.
\]
This polynomial is increasing and convex, so Jensen and Proposition~\ref{prop:bound} give
\[
\mathbb E_\mu[\mathrm{MAL}(\alpha_t,K)]
\ge\mathrm{MAL}(\mathbb E_\mu[\alpha_t],K)
\ge\mathrm{MAL}([1-\bar A-\delta]_+,K).
\]
This counts the correction or bonus token; accepted draft count is one less.
Proposition~\ref{prop:block} avoids the geometric model through reached-stage laws.

\section{Extended related work}
\label{app:related}

\textbf{Windowed drafting at moderate context.}
\citet{sadhukhan2025magicdec} show speculative decoding remains profitable at large batch sizes for long sequences using drafters with sparse {KV} caches, including a small drafter with a {StreamingLLM}-style windowed cache.
We share the windowing insight but operate where they do not: their draft {KV} budgets are 256--2049 tokens at contexts up to 100k, within the models' native rotary range, so the position problem of \secref{sec:intro} never arises and no position mechanism is needed.
Their analysis is also driven by batch size, whereas we run at batch size~1 and one to two orders of magnitude longer.
We adopt their cost decomposition for our own accounting.

\textbf{Streaming and cache-relative positioning.}
\citet{xiao2024streamingllm} enable unbounded streaming by pairing a rolling window with pinned attention sinks, assigning positions by index \emph{within the cache} and, for {RoPE}, rotating the cache at each decoding step.
The relevant distinction is whether a key's assigned index is stable for its lifetime.
Theirs is cache-relative, so a one-token slide shifts every local index and the whole window must be re-rotated; ours is text-relative up to a constant per index band, so a key is rotated once at cache entry and never touched again.
Cache-relative positioning is sound and we do not claim to be simply cheaper: the trade is their recurring re-rotation against our possible boundary effects (\secref{sec:rebasing}).
A quantitative output comparison at those boundaries, and a wall-clock comparison accounting for fused-kernel and bandwidth effects, remain future work.

\textbf{Trained long-context drafters.}
\citet{yang2026longspec} train a drafter with sliding-window self-attention and cache-free cross-attention, with anchor-offset indices to generalize position encoding from short training sequences; \citet{wang2026specla} train an {EAGLE}-style drafter and address recurrent-state rollback for linear-attention targets.
Both solve the position-generalization problem we solve, but through training.
Our contribution is that no \emph{long-context} training is required.

\textbf{Native heads, retrieval drafting, and benchmarks.}
Windowed-{MTP}~\citep{valliappan2026windowed} is closest in mechanism, being training-free and windowing draft attention while verification stays at full context, but it windows a native multi-token-prediction head already inside the target checkpoint, so targets without such a head cannot use it.
We window an \emph{external}, separately released drafter, which is what forces the position problem: a native head shares the target's positional regime and an external short-context drafter does not.
Retrieval- and $N$-gram-based drafters~\citep{saxena2023pld,fu2024lookahead,ou2024anpd,he2024rest} also exploit local continuation structure but require an exact prior occurrence; the same acceptance inequalities analyze them once a comparator is fixed.
Benchmarks built specifically for speculative decoding remain predominantly short-context~\citep{xia2024specbench}; {SPEED-Bench}~\citep{abramovich2026speedbench} reaches 32k with third-party entropy tiers, which we use for the gap measurements of \secref{sec:gaps} and the window ablation of \secref{sec:ablations}, though it cannot reach the context lengths of \secref{sec:e4}.
General long-context benchmarks go further and we use {LongBench}~\citep{bai2024longbench,bai2025longbench2} and {RULER} directly (Appendices~\ref{app:datasets} and~\ref{app:niah}), but they are not instrumented for acceptance or emitted length, so the per-round quantities of \secref{sec:experiments} must be collected by us either way.

\section{Why the comparison is internal}
\label{app:baselines}

This paper reports no external speculative-decoding system as a baseline.  The reason is not that none was sought, but that the operating point the paper is about (an externally released and \emph{frozen} short-context drafter, verification left exact against the complete prefix, single-stream decoding, and context to 2M tokens) is one at which each published alternative is either undefined or defined only after doing the thing whose necessity is in question.  We give the obstruction method by method, then say what our arms do in place of an external comparison.

\textbf{Methods that require training.}
\citet{yang2026longspec} and \citet{wang2026specla} train long-context drafters, and \citet{tan2025specpv} repairs a released {EAGLE-3} head's positional range by {YaRN} fine-tuning at 32k.
Matching any of them at our lengths means first producing a trained drafter for each of our three targets.  The resulting number would measure the value of that training, which is the very quantity our claim concerns, and it would no longer hold the drafter fixed, so the comparison would answer a different question from the one asked.
{TokenSwift}~\citep{wu2025tokenswift} likewise attaches trained draft heads, and {TriForce}~\citep{sun2024triforce} requires a standalone draft model, which for a 2.8T-parameter target no released checkpoint supplies.

\textbf{Methods that require a property our drafters lack.}
Windowed-{MTP}~\citep{valliappan2026windowed} is training-free and windows draft attention as we do, but it windows a multi-token-prediction head already resident inside the target checkpoint.
That head shares the target's positional regime, which is both why it has no position problem to solve and why the method does not transfer to an external drafter carrying its own shorter rotary reach.  The mechanism is close; the precondition is one our setting does not meet.

\textbf{Methods that change what is verified.}
{SpecPV}~\citep{tan2025specpv} verifies against a sparse {KV} cache with periodic full refreshes, and {RAPID}~\citep{chen2025rapid} enriches the target distribution by retrieval.
Both accept a departure from the full-context target's own prediction rule, where we leave verification exact and act only on the drafter's view.  A speedup obtained under a different output law is not commensurable with one obtained under ours, and the two directions compose rather than compete, since a cheaper verifier and a cheaper drafter address different terms of the same round.

\textbf{Methods that are not target-conditioned.}
Retrieval- and $N$-gram drafters~\citep{saxena2023pld,fu2024lookahead,ou2024anpd,he2024rest} propose from surface repetition rather than from target features.
They are training-free and would run at 2M, but they do not bear on the question here, which is whether a short-context \emph{model} can serve a long-context target.  They answer a different question, by not using a model.

\textbf{Methods whose operating point is shorter.}
\citet{sadhukhan2025magicdec} use draft {KV} budgets of 256--2049 tokens at contexts to 100k, inside the models' native rotary range; {SpecPV} reports to 60k, and {TokenSwift} targets sequences beyond 20k.
None of these lengths reaches $\Pmax$, so none encounters the failure of \secref{sec:position}, and a comparison conducted at their lengths would be conducted where our second mechanism is provably the identity map (Corollary~\ref{cor:rebasing}) and therefore has nothing to show.

\textbf{What the arms do instead.}
The nearest published configurations are consequently not missing from the tables; they are in them.
The \emph{unwindowed} arm is a released {DSpark} or {EAGLE} drafter run exactly as shipped, which is what a practitioner obtains by default and the thing the method has to beat.
The \emph{windowed-only} arm is the bounded-draft-cache idea of \citet{sadhukhan2025magicdec} carried to our lengths, where, as in their own regime, no position mechanism is applied.
Because all three arms run on one engine build, one node, and one prompt set, and the latter two differ by a single environment flag, each contrast isolates one factor.  A cross-paper comparison against differently built systems on different hardware could not do that, and \secref{sec:setup} shows the margin at stake in such confounds, a single asymmetry in compiled-graph capture moving a ratio from $0.38\times$ to $2.21\times$.

\textbf{The one alternative we do not run.}
{StreamingLLM}'s cache-relative indexing~\citep{xiao2024streamingllm} is a genuine alternative to re-basing rather than to windowing, and we do not run it.
The trade is stated in \secref{sec:rebasing} and above: their assigned positions are stable within the cache and ours within the text, so they re-rotate the visible cache at every decode step where we rotate once at entry, and we accept possible band-boundary effects where they do not.
Settling which is faster in wall-clock terms needs a fused-kernel and bandwidth accounting we have not done, and we do not claim to be simply cheaper.

\section{Additional experimental tables}
\label{app:tables}

\textbf{Qualifications on the {Kimi-K3} ladder} (\secref{sec:e4}).
Appendix~\ref{app:more} records the repetition problem on forced {PG19} decoding and the reruns that reproduce the cliff with loops removed.

\textbf{Qualifications on the {Qwen} ladder} (\secref{sec:e1}).
The 1M column of Table~\ref{tab:e1} rests on five samples collected across three part-runs and has the widest spread, and its target denominator was collected against a checkpoint with Dual Chunk Attention ({DCA}) enabled while the windowed arm used the noDCA variant.
A revalidation on a pinned engine build, with the target arm re-measured on the matched checkpoint, reproduces both quantities: {MAL} 1.659, 1.693, 1.610, and 1.609 at 2k, 8k, 32k, and 128k against 1.691, 1.766, 1.679, and 1.685 originally, every difference inside one standard error at per-sample standard deviations of 0.085--0.167, with matched speedups of $1.29\times$ to $1.58\times$ still rising with context.
An exploratory sweep over $K\in\{1,3,5\}$ finds $K=3$ at or near the throughput optimum everywhere, with the $K{=}3\to K{=}5$ change flipping sign with context, consistent with draft cost being context-independent under a fixed window while verification cost is not.

\textbf{Sink ablation} (\secref{sec:ablations}).
A paired ablation on {Inferact} at $s\in\{0,128\}$ over four tiers finds the sign of the {MAL} difference alternating and the mean favoring $s=0$, {MAL} 2.19 against 1.96 and $S$ $1.47\times$ against $1.31\times$, so the sink is unused capacity for this drafter rather than a required stabilizer.
The arms were separate engine boots with enough within-arm spread that only the 1M tier is unambiguous, so this justifies a configuration choice rather than a finding about attention sinks in general.

\begin{table}[h]
\caption{{Qwen2.5-7B-Instruct-1M} with its {EAGLE-3} head on {PG19}, $W=2048$, $K=3$, $n=5$.  The 1M column ($\dagger$) is excluded from every claim in the paper, its target denominator having been collected against a {DCA}-enabled checkpoint while the windowed arm used the noDCA variant; \Secref{sec:e1} reads the ladder to 512k.}
\label{tab:e1}
\begin{center}
\small
\begin{tabular}{lrrrrrrrrrr}
\toprule
& 2k & 4k & 8k & 16k & 32k & 64k & 128k & 256k & 512k & 1M$^\dagger$\\
\midrule
windowed {MAL}   & 1.69 & 1.63 & 1.77 & 1.74 & 1.68 & 1.70 & 1.69 & 1.81 & 1.61 & 2.31\\
windowed $S$     & 1.34 & 1.31 & 1.46 & 1.49 & 1.51 & 1.60 & 1.65 & 1.83 & 1.64 & 2.39\\
unwindowed {MAL} & 1.50 & 1.41 & 1.36 & 1.34 & 1.31 & 1.25 & 1.28 & 1.31 & 1.23 & 1.44\\
unwindowed $S$   & 1.22 & 1.16 & 1.13 & 1.14 & 1.14 & 1.11 & 1.13 & 1.19 & 1.11 & 1.34\\
\bottomrule
\end{tabular}
\end{center}
\end{table}

\begin{table}[h]
\caption{Paired window ablation, {Qwen2.5-7B-Instruct} with {EAGLE-3} on {SPEED-Bench}, $K=3$, $\approx$480 samples per tier (\secref{sec:ablations}).  The 2k tier is omitted: its speedups are out of line with both neighbours and its cell counts are irregular.}
\label{tab:e6}
\begin{center}
\small
\begin{tabular}{rrrrrrr}
\toprule
& \multicolumn{2}{c}{unwindowed} & \multicolumn{2}{c}{$W=2048$} & \multicolumn{2}{c}{$W=1024$}\\
\cmidrule(lr){2-3}\cmidrule(lr){4-5}\cmidrule(lr){6-7}
$N$ & {MAL} & $S$ & {MAL} & $S$ & {MAL} & $S$\\
\midrule
512 & 1.75 & 1.38 & 1.75 & 1.34 & 1.75 & 1.37\\
1k  & 1.72 & 1.36 & 1.72 & 1.33 & 1.70 & 1.32\\
4k  & 1.42 & 1.13 & 1.68 & 1.32 & 1.68 & 1.35\\
8k  & 1.33 & 1.08 & 1.63 & 1.32 & 1.64 & 1.36\\
16k & 1.27 & 1.06 & 1.63 & 1.38 & 1.64 & 1.39\\
32k & 1.26 & 1.09 & 1.67 & 1.48 & 1.68 & 1.49\\
\bottomrule
\end{tabular}
\end{center}
\end{table}

\begin{table}[h]
\caption{Greedy agreement patterns at matched corpus positions, fresh-drafter law (\secref{sec:e2delta}).  Cancellation is the $x^p=x^q\ne x^{\mathrm{trunc}}$ row of Appendix~\ref{app:greedy}.}
\label{tab:patterns}
\begin{center}
\small
\begin{tabular}{rrrrrr}
\toprule
$W$ & all agree & trunc.\ only $+$ all distinct & fidelity only & cancellation & greedy agreement\\
\midrule
128  & .3109 & .2143 & .4423 & .0326 & .3434\\
256  & .3246 & .1749 & .4697 & .0309 & .3554\\
512  & .3206 & .1383 & .5109 & .0303 & .3509\\
1024 & .3223 & .1234 & .5314 & .0229 & .3451\\
2048 & .3251 & .1023 & .5543 & .0183 & .3434\\
8192 & .3194 & .0697 & .5983 & .0126 & .3320\\
\bottomrule
\end{tabular}
\end{center}
\end{table}
\section{Additional model pairs and ablations}
\label{app:more}

This appendix collects the model pairs, corpora, and configurations that \secref{sec:experiments} summarizes but does not tabulate.  All of them use the protocol of \secref{sec:setup}: frozen drafters, no weight updates, greedy decoding, batch size~1 unless stated, decode-phase timing, and speedup computed per sample against a paired target-only run on the identical prompt.

\subsection{A third pair: {MiniMax-M3} with {EAGLE-3.1}}

\begin{table}[h]
\caption{{MiniMax-M3} with a released {EAGLE-3.1} head on {PG19}, $K=3$, $W=4096$, $n=5$ per tier, 4096 generated tokens with {EOS} ignored, {TP}=8.}
\label{tab:m3}
\begin{center}
\footnotesize
\begin{tabular}{rrrrr}
\toprule
& \multicolumn{2}{c}{unwindowed} & \multicolumn{2}{c}{windowed, $W{=}4096$}\\
\cmidrule(lr){2-3}\cmidrule(lr){4-5}
context & {MAL} & $S$ & {MAL} & $S$\\
\midrule
4k   & 3.24 & $2.38\times$ & 3.25 & $2.37\times$\\
8k   & 2.81 & $2.09\times$ & 2.82 & $2.08\times$\\
16k  & 3.02 & $2.23\times$ & 3.07 & $2.25\times$\\
32k  & 2.94 & $2.15\times$ & 3.02 & $2.21\times$\\
64k  & 3.19 & $2.32\times$ & 3.32 & $2.42\times$\\
128k & 2.87 & $2.10\times$ & 3.08 & $2.29\times$\\
256k & 2.54 & $1.84\times$ & 2.80 & $2.11\times$\\
512k & 2.56 & $1.85\times$ & 2.83 & $2.18\times$\\
1M   & 2.49 & $1.75\times$ & 3.32 & $2.62\times$\\
\bottomrule
\end{tabular}
\end{center}
\end{table}

Three features of Table~\ref{tab:m3} matter for the argument.
First, the target does not collapse: {MiniMax-M3} uses block-sparse attention and loses only 26\% of its decode rate from 4k to 1M, against 32\% for {Kimi-K3} over the same span, so the windowed arm's advantage cannot be attributed to a baseline that degrades underneath it.
Retention states the same result directly: the windowed arm ends at $1.11$ of its 4k speedup and the unwindowed arm at $0.74$.
Second, below 64k the two arms are within $0.06\times$ of each other and unwindowed is nominally ahead at 4k and 8k; the gap opens from 128k and reaches 50\% relative at 1M.
This is the same shape as Table~\ref{tab:e4} and the same shape as the window ablation of \secref{sec:ablations}, from a third stack.  Windowing here is block-table compaction, a genuine {KV} reduction rather than a mask over a full-length cache, and was verified live in the engine.  A third arm that windows without re-basing is omitted because it agrees with the re-based arm to within $0.05\times$ at every tier, as it must: this drafter's rotary reach of $1{,}048{,}576$ exceeds the deepest position served, so the pair cannot test re-basing.
Third, the mechanism here is acceptance rather than draft cost: the per-round cost in plain-decode units, $\mathrm{MAL}/S$, stays at 1.27--1.37 for the windowed arm at every tier, so a $K{+}1=4$ token verification costs about 30\% more than a single decode and speedup reduces to acceptance divided by that constant.
The drafter's own attention is full {MHA} against the target's 16:1 {GQA}, so its single layer would hold 32~GiB of {KV} at 1M, more than all sixty target layers combined, against 128~MiB windowed; on this pair windowing buys memory and acceptance at once.

\textbf{The verify-path condition.}
An earlier complete sweep of this pair on the engine's default Triton attention backend is retracted for its speedup column and is instructive.
It reported the opposite qualitative result, $0.74\times$ at 1M, with {MAL} within the prompt-level spread of the numbers above.
The diagnostic is the per-round cost $\mathrm{MAL}/S$, which rose from 1.41 at 4k to 4.31 at 1M, and a value of $K{+}1$ means a four-token verification is costing four separate decodes, at which point no drafter and no $K$ can help.
A decode-only trace attributed 82\% of {GPU} time to one kernel, the three dense layers at the bottom of the otherwise sparse target, which cost $4.2\times$ as much for $4\times$ the query rows, so verification was not amortizing at all; the drafter's own forward measured 0.123~ms and was never the problem.
Switching to a correctly parallelized attention kernel raised the target arm as well, by 82\% at 1M, and restored the ratio.
This is also the cleanest illustration of \secref{sec:failures}: acceptance was essentially unchanged while the speedup verdict flipped sign, so an acceptance bound does not determine speedup.
We record it because a kernel configuration masqueraded as an algorithmic result, and because it gives the condition a practical form: track $\mathrm{MAL}/S$, expect a small constant independent of context, and treat drift in it as a systems fault rather than a property of drafting.

\subsection{Re-basing ablation}
\label{app:rebase}

\begin{table}[h]
\caption{Matched toggle of modulus re-basing on {Kimi-K3}, $K=3$, $W=4096$, forced 4096-token generations, at 1M and 2M and extended to further corpora.  The patch is present in both arms and only an environment flag differs, so binary, hardware, prompts, and drafter are identical.}
\label{tab:rebase}
\begin{center}
\footnotesize
\begin{tabular}{llrrrr}
\toprule
& & \multicolumn{2}{c}{without re-basing} & \multicolumn{2}{c}{with re-basing}\\
\cmidrule(lr){3-4}\cmidrule(lr){5-6}
drafter & corpus, context & {MAL} & $S$ & {MAL} & $S$\\
\midrule
{Inferact} & {PG19}, 1M (control) & 2.92 & $1.91\times$ & 2.95 & $1.92\times$\\
{Inferact} & {PG19}, 2M           & 1.06 & $0.68\times$ & 3.18 & $2.05\times$\\
{RadixArk} & {PG19}, 1M (control) & 2.50 & $1.27\times$ & 2.50 & $1.50\times$\\
{RadixArk} & {PG19}, 2M           & 1.08 & $0.71\times$ & 2.56 & $1.68\times$\\
\midrule
{Inferact} & {LongBench}, 2M      & 1.12 & $0.86\times$ & 3.41 & $2.62\times$\\
{RadixArk} & {LongBench}, 2M      & 1.18 & $0.93\times$ & 2.73 & $2.13\times$\\
{Inferact} & {BEAM}, 2M           & 1.14 & $0.87\times$ & 2.29 & $1.71\times$\\
{RadixArk} & {BEAM}, 2M           & 1.16 & $0.91\times$ & 2.14 & $1.66\times$\\
\bottomrule
\end{tabular}
\end{center}
\end{table}

Table~\ref{tab:rebase} is the direct test of the failure mode \secref{sec:position} predicts.
The implementation gathers rotary factors from a table of exactly $\Pmax=1{,}048{,}576$ rows with no bounds check.  Both checkpoints declare that figure as their maximum position and reach it by {YaRN} scaling of a shorter native range, $32{,}768\times32$ for {Inferact} and $65{,}536\times16$ for {RadixArk}, so 1M is the drafters' stated limit and not an incidental allocation; the {Kimi-K3} target is {NoPE} and serves contexts past it while still handing the drafter true absolute positions.
The prediction is a cliff rather than a decay, because the failure is an out-of-range table access rather than extrapolation, and that is what the data show: identical behavior at 1M, where the deepest position is $1{,}048{,}575$, and a collapse to {MAL} near one the moment the context exceeds the table.
Three consequences are worth stating.
The failure is silent in exactly the sense of \secref{sec:position}: the verifier still rejects the bad proposals, so the text stays correct while the speedup goes below $1\times$ and speculation becomes a tax.
Re-basing is provably inert below $\Pmax$, and the control rows confirm it empirically rather than by argument.
And the fix rescues only the windowed drafter, as \secref{sec:strand-one} requires: an unwindowed draft's visible span straddles the band boundary, so no single shift preserves every offset, and it is not recovered.
The {RadixArk} 1M control shows a $0.23\times$ difference in $S$ at equal {MAL}, which is run-to-run throughput variation on that arm rather than an effect of the flag; the {MAL} equality is the quantity the control is about.
The same inertness holds across the ladder of Table~\ref{tab:e4}, where re-basing is nominally ahead in seven of the twelve tiers through 1M and windowing alone in five, the differences staying inside the per-cell spread quoted in \secref{sec:e4}.  Only an environment flag separates the arms on one binary, so this is the expected reading of a control rather than a null result to explain away.

\subsection{Beyond {PG19}: {LongBench}, {BEAM}, and {SPEED-Bench}}
\label{app:datasets}

The first four corpora of Table~\ref{tab:datasets} run the seven tiers from 8k to 2M, each scored on its most trustworthy decode condition: forced 4096 tokens where continuation has no natural stopping point, the span before the first {EOS} where the model answers and stops, and {EOS}-stopped generation for the {BEAM} rerun built to allow it.  The {BEAM}-prefix 2M cells are the one visible dip and are confounded twice, that tier alone changing both prompt construction and source data.

\begin{table}[h]
\caption{Generality of the flatness result of \secref{sec:e4} across corpora, every cell the windowed and re-based arm on {Kimi-K3} at $K=3$, $W=4096$, batch 1 under one shared setup.  The $S$ range column spans all measured tiers; retained is a tier's speedup divided by the same arm's speedup at 8k.  {PG19} is the ladder of Table~\ref{tab:e4}.}
\label{tab:datasets}
\begin{center}
\footnotesize
\setlength{\tabcolsep}{4pt}
\begin{tabular}{llcrrcrr}
\toprule
& & \multicolumn{3}{c}{{Inferact}} & \multicolumn{3}{c}{{RadixArk}}\\
\cmidrule(lr){3-5}\cmidrule(lr){6-8}
& & & \multicolumn{2}{c}{retained at} & & \multicolumn{2}{c}{retained at}\\
\cmidrule(lr){4-5}\cmidrule(lr){7-8}
corpus & decode & $S$ range & 1M & 2M & $S$ range & 1M & 2M\\
\midrule
{PG19} books        & forced          & $2.10$--$2.55$ & $1.02$ & $0.90$ & $1.63$--$2.09$ & $0.93$ & $1.19$\\
{LongBench} v1+v2   & forced          & $2.29$--$2.62$ & $1.11$ & $1.13$ & $1.23$--$2.18$ & $1.30$ & $1.27$\\
{BEAM} prefixes     & pre-{EOS}       & $1.71$--$2.52$ & $1.06$ & $0.72$ & $1.66$--$2.46$ & $1.01$ & $0.68$\\
{BEAM} v2 chats     & {EOS}-stopped   & $1.77$--$2.48$ & $1.28$ & $0.96$ & $1.68$--$2.47$ & $1.24$ & $0.88$\\
{SPEED-Bench} 1k--32k & forced        & $1.69$--$2.48$ & --     & --     & $1.61$--$2.16$ & --     & --\\
\bottomrule
\end{tabular}
\end{center}
\end{table}

The {Kimi-K3} ladder of \secref{sec:e4} uses {PG19} book continuations.  Four further corpora repeat the comparison with the same seven-arm protocol, spanning natural documents, long conversations in two renderings, and a short-context throughput split (Table~\ref{tab:datasets}).  They exist to answer one objection: that flatness is a property of book prose, which is unusually uniform and therefore unusually easy to draft.  Flatness is what carries over: every corpus that reaches 1M retains $0.93$--$1.30$ of its 8k speedup there on both drafters, and the two failure modes of \secref{sec:e4} reappear in the same shapes on each of the three long-context corpora.  The absolute level is not invariant and is not claimed to be; it is set by how predictable the corpus is.

On \textbf{{LongBench}}, with {LongBench}-v1 documents at 8k--32k and deduplicated {LongBench}-v2 documents at 128k--2M and no document reused across tiers, windowed {Inferact} holds $2.29$--$2.62\times$ flat from 8k to 2M and windowed {RadixArk} $1.68$--$2.18\times$ apart from a single dip to $1.23\times$ on the 32k v1 documents, where {Inferact} holds $2.55\times$, while the unwindowed arms fail in the two distinct ways of \secref{sec:e4}: {Inferact}'s acceptance collapses, {MAL} 3.35 at 8k falling to 1.07 at 2M with $S$ from $2.56\times$ to $0.56\times$, whereas {RadixArk} holds {MAL} 2.2--2.5 to 1M and still slides to $0.36\times$ and then $0.14\times$ because the full-{KV} read dominates each draft step.
That the same two failure modes reappear on natural documents, against the same target, is the strongest available evidence that the effect is not a property of {PG19} prose.

On \textbf{{BEAM}}~\citep{tavakoli2025beam} long conversations, measured on the answer span before the first {EOS} because forced continuation past {EOS} degenerates into repetition that is trivially draftable, both windowed drafters hold $2.2$--$2.5\times$ from 8k to 1M and $1.66$--$1.71\times$ at 2M, while the unwindowed arms again split into acceptance collapse ({Inferact}, {MAL} 1.11 at 2M) and throughput collapse ({RadixArk}, $0.14\times$ at 2M at {MAL} 1.35).
The 2M tier is the one place where flatness visibly dips, to $0.72$ and $0.68$ retention, and it is the least trustworthy cell in the study: that tier alone switches both prompt construction and source data, so its drop cannot be attributed to length.

A second rendering of the same conversations, \textbf{{BEAM} v2}, is the cleanest reference available and exists to remove both objections to the first: every prompt is rendered with the chat template at every tier, so it ends in a real user turn rather than mid-conversation, and generation stops at the model's own {EOS}, so no output contains a forced-continuation loop.  The target emits {EOS} on every sample, with median answers of 458 to 2140 tokens.  Both windowed drafters hold $1.68$--$2.48\times$ across the ladder and retain $0.96$ and $0.88$ of their 8k speedup at 2M, while the unwindowed arms split exactly as before, {Inferact} to {MAL} 1.11 at $0.58\times$ and {RadixArk} to $0.14\times$ at {MAL} 1.32.  That the result survives with loops removed entirely matters more than its level, because forced decode is the one measurement artifact that could have manufactured it.

\textbf{{SPEED-Bench}} is not a long-context test and is not offered as one: its throughput split runs 1k to 32k on untruncated prompts of 1.1k to 36.9k tokens.  It is here for the two things the long-context corpora cannot provide.  The first is coverage of other domains and harder tasks: three task families, code completion, creative writing, and book continuation, each split into a low- and a high-entropy tier, the only place in the study where the drafters work outside long documents and chats.  The second is a clean measurement of what the window itself costs.  From the 8k bucket upward the window is a strict subset of the prompt, about one eighth of it at 32k, so the unwindowed drafter reads several times what the windowed one does while re-basing is inert throughout, every position being far below $\Pmax$.  Across the twenty matched cells the bounded view costs almost nothing: windowed {Inferact} is at most $0.16\times$ behind its unwindowed counterpart, and windowed {RadixArk} is ahead in nine of ten buckets, by up to $0.38\times$.  In acceptance the gap is at most 0.22 {MAL} out of roughly 3.2 on low-entropy code, and within 0.04 on high-entropy prose, where the drafter has least to go on.  Every arm wins, $1.41$--$2.57\times$, and what separates the cells is content rather than configuration, low-entropy code drafting at $2.3$--$2.6\times$ against $1.4$--$2.1\times$ for high-entropy prose, consistent with the entropy-tier effect in Appendix~\ref{app:gapdetail}.

\subsection{Retrieval under a bounded drafter view}
\label{app:niah}

This is the study's stress test of the method on long-context recall.  Everywhere else the task is continuation, where the next token is usually inferable from recent tokens and a bounded drafter view is plausibly harmless by construction.  A needle task is built the other way: the answer depends on one short span placed anywhere in a long context, which is exactly the span a windowed drafter cannot see.  If bounding the view were going to cost accuracy anywhere, it would be here, and the question is whether verification really protects what the system retrieves.
On {RULER} single-needle tasks against {Kimi-K3} with windowed {Inferact}, needles inserted at depths from 10\% to 90\% and five samples per length from 8k to 1M, both the target and the windowed arm answer 120 of 120 correctly on each of the three task variants (Table~\ref{tab:niah}).
Speedup instead tracks how predictable the haystack is: $2.41\times$ mean on the synthetic noise haystack, where the drafter's window is highly predictable, and $1.54\times$ and $1.38\times$ on natural-prose haystacks with numeric and {UUID} needles.
The needle is a short span that usually sits outside the window, and the drafter loses acceptance on the surrounding prose rather than on the answer.
Because accuracy is identical on both arms at every cell while {MAL} varies by a factor of two across variants, this is the cleanest separation in the paper between what a bounded drafter view costs, which is acceptance, and what it does not cost, which is what the system retrieves.

\begin{table}[h]
\caption{Needle-in-a-haystack retrieval on {RULER}'s three single-needle variants, {Kimi-K3} with windowed {Inferact}, $W=4096$, $s=0$, $K=3$.  Needles sit at depths from 10\% to 90\%, five samples per cell, greedy 48-token answers scored by separator-insensitive normalized substring match.  Four representative lengths of the eight measured are shown.}
\label{tab:niah}
\begin{center}
\footnotesize
\begin{tabular}{rrrcc}
\toprule
& & & \multicolumn{2}{c}{answers correct}\\
\cmidrule(lr){4-5}
$N$ & $S$ & {MAL} & target & windowed\\
\midrule
\multicolumn{5}{l}{\emph{single-1}: synthetic-noise haystack, numeric needle (mean $S=2.41\times$)}\\
8k   & $2.24\times$ & 3.37 & 5/5 & 5/5\\
128k & $2.26\times$ & 3.02 & 5/5 & 5/5\\
512k & $2.67\times$ & 3.48 & 5/5 & 5/5\\
1M   & $2.47\times$ & 3.28 & 5/5 & 5/5\\
\midrule
\multicolumn{5}{l}{\emph{single-2}: natural-prose haystack, numeric needle (mean $S=1.54\times$)}\\
8k   & $1.97\times$ & 2.63 & 5/5 & 5/5\\
128k & $1.69\times$ & 2.22 & 5/5 & 5/5\\
512k & $1.39\times$ & 1.80 & 5/5 & 5/5\\
1M   & $1.24\times$ & 1.61 & 5/5 & 5/5\\
\midrule
\multicolumn{5}{l}{\emph{single-3}: natural-prose haystack, {UUID} needle (mean $S=1.38\times$)}\\
8k   & $1.59\times$ & 2.14 & 5/5 & 5/5\\
128k & $1.43\times$ & 1.89 & 5/5 & 5/5\\
512k & $1.33\times$ & 1.73 & 5/5 & 5/5\\
1M   & $1.25\times$ & 1.61 & 5/5 & 5/5\\
\bottomrule
\end{tabular}
\end{center}
\end{table}

\subsection{Batched serving}
\label{app:batch}

Every timing result in \secref{sec:experiments} is batch size~1, and Table~\ref{tab:batch} re-runs the comparison on a batched engine at the same 4096-token decode length, sweeping batch sizes 1, 2, 4, and 8 against contexts from 4k to 800k.
Windowed {Inferact} stays above break-even in all 19 measured cells, $1.14$--$2.45\times$, while unwindowed falls below it in 10 of 19, from roughly 131k onward and down to $0.64\times$.
The ordering of \secref{sec:e4} survives batching from 131k onward, which is where the two arms separate at batch~1 as well, but it does not hold below that: at 4k the unwindowed arm is ahead in three of the four batch sizes, by as much as $0.57\times$ at $B=1$, since at that length it sees little more than the window and pays almost nothing for the difference.  The intermediate rungs cannot be compared at all, the two arms having been tabulated at 32k and 65k respectively.  Where the ordering does hold, the margin does not grow with batching: target throughput scales close to linearly in $B$, so plain batched decode already recovers much of the memory-bandwidth surplus speculation exploits, and the windowed advantage decays toward break-even as $B$ rises even though acceptance is roughly flat in $B$.
At 4k the windowed arm slides from $1.67\times$ to $1.14\times$ across $B=1$ to $8$.
Two artifacts had to be removed before these numbers were stable and both are worth recording, since both inverted a conclusion: capture sizes that covered the sequence count but not the verify query width $B(K{+}1)$ sent every verify forward down an eager path and manufactured a $0.24$--$0.47\times$ collapse at $B\ge4$, and firing concurrent requests at near-maximal {KV} occupancy serializes them behind staggered prefills rather than co-decoding.  The protocol is therefore one engine per arm with prefix caching, each of the $B$ distinct {PG19} prompts prefilled once before all $B$ are fired concurrently, so every cell is a pure co-decode.
Acceptance on this workload is loop-inflated and must not be read from these runs.  {EOS} essentially never fires on {PG19} continuation, 1 of 625 rows, and greedy decode degenerates into repetition at many contexts, the target included, so on non-degenerate cells {MAL} is 2.24 windowed against 1.27 unwindowed, while the degenerate average pulls the unwindowed figure up to 2.57 and masks the collapse.  Speedup stays a fair ratio because the target decodes the same degenerate text, but the windowed 524k and 800k cells are loop-assisted; a three-repetition recheck confirms them reproducible ($2.47\to2.44$ at 524k).

\begin{table}[h]
\caption{Batched serving on {Kimi-K3} with {Inferact}, $W=4096$, $K=3$, 4096 generated tokens per request, natural {EOS}.  Entries are decode-phase speedup over the batched target at the same batch size and context, so each column is its own baseline.  Dashes are combinations whose working set exceeds the {KV} budget, which the engine cannot allocate; they are not skipped measurements.}
\label{tab:batch}
\begin{center}
\footnotesize
\begin{tabular}{rrrrr}
\toprule
& \multicolumn{4}{c}{batch size $B$}\\
\cmidrule(lr){2-5}
$N$ & 1 & 2 & 4 & 8\\
\midrule
\multicolumn{5}{l}{\emph{{Inferact}, windowed}}\\
4k   & 1.67 & 1.75 & 1.44 & 1.14\\
32k  & 2.14 & 1.86 & 1.78 & 1.42\\
131k & 1.74 & 1.77 & 1.71 & 1.56\\
262k & 1.70 & 1.82 & 1.77 & --\\
524k & 2.44 & 2.45 & --   & --\\
800k & 2.19 & 2.08 & --   & --\\
\midrule
\multicolumn{5}{l}{\emph{{Inferact}, unwindowed}}\\
4k   & 2.24 & 1.65 & 1.49 & 1.33\\
65k  & 1.81 & 1.71 & 1.44 & 1.51\\
131k & 0.94 & 1.15 & 0.94 & 0.68\\
262k & 0.82 & 0.73 & 0.64 & --\\
524k & 0.78 & 0.68 & --   & --\\
800k & 0.69 & 0.64 & --   & --\\
\bottomrule
\end{tabular}
\end{center}
\end{table}

\subsection{Sink-size ablation}
\label{app:sinkabl}

Windowed drafters are often paired with \emph{attention sinks}, a few pinned prefix tokens the drafter attends to in addition to its recent window~\citep{xiao2024streamingllm}.  Table~\ref{tab:sink} varies the sink budget on {Inferact} with everything else fixed, at 1M context, $W=4096$, and $n=5$ per point.

\begin{table}[h]
\caption{Sink-size ablation on {Inferact} against {Kimi-K3}, at 1M context with $W=4096$, $K=3$, and $n=5$ per point.  Only the number of pinned prefix tokens varies.  {MAL} is emitted tokens per round out of $K+1=4$.}
\label{tab:sink}
\begin{center}
\small
\begin{tabular}{lrr}
\toprule
drafter & sink tokens $s$ & {MAL}\\
\midrule
{Inferact} & 0  & \textbf{3.26}\\
{Inferact} & 4  & 2.12\\
{Inferact} & 8  & 2.01\\
{Inferact} & 16 & 2.16\\
{Inferact} & 32 & 2.25\\
\bottomrule
\end{tabular}
\end{center}
\end{table}

No sink is decisively best, and the penalty for adding one is large and immediate: any sink at all costs more than a full point of {MAL}, and the loss does not shrink as the budget grows from 4 to 32.  This is the opposite of what sinks do for a streaming language model, where pinning the first few tokens stabilizes generation.  The explanation we favor is positional: while decoding near position 1M, a sink token sits at absolute position $\approx 0$, so the drafter is asked to attend across an offset it never saw in training, and the resulting features are worse than no distant context at all.  A target-conditioned drafter also has less need of a sink than a standalone one, because whatever the prefix contributes is already present in the target features it reads.  This ablation, together with the paired $s\in\{0,128\}$ sweep on the same drafter in Appendix~\ref{app:tables}, is what justifies $s=0$ on {Inferact}.

On {RadixArk} it does not, and we would rather say so than generalize: that drafter responds in the opposite direction on a single sweep, so the sink budget may be drafter-specific, and we report the ablation on {Inferact} alone until a re-run settles the direction.  We nonetheless run {RadixArk} at $s=0$ throughout, for comparability with {Inferact} rather than because this ablation selects that value for it.  No result in the paper turns on the choice, since every arm of every comparison shares the same sink budget; what is not established is that $s=0$ is the best setting for {RadixArk}.

\subsection{Window-size ablation at long context}
\label{app:winabl}

Section~\ref{sec:ablations} varies $W$ at 32k on the {Qwen} pair (Table~\ref{tab:e6}).  Table~\ref{tab:win} repeats the question thirty times further out, at 1M context on {Kimi-K3} with $s=0$ and $n=5$ per point, varying only {Inferact}'s window.

\begin{table}[h]
\caption{Window-size ablation on {Inferact} against {Kimi-K3}, at 1M context with $s=0$, $K=3$, and $n=5$ per point.  Only the drafter's visible window varies.  {MAL} is emitted tokens per round out of $K+1=4$.}
\label{tab:win}
\begin{center}
\small
\begin{tabular}{lrr}
\toprule
drafter & window $W$ & {MAL}\\
\midrule
{Inferact} & 1024 & 2.97\\
{Inferact} & 2048 & 2.94\\
{Inferact} & 4096 & 3.03\\
\bottomrule
\end{tabular}
\end{center}
\end{table}

Acceptance is flat: the spread across a $4\times$ range of window sizes is 0.09 {MAL}, smaller than the sample spread within any one setting.  Two things follow.  First, the saturation found at $W\ge1024$ on the {Qwen} pair at 32k is not an artifact of that pair or that length; the most recent kilotoken carries essentially all of this drafter's usable signal even when a million tokens precede it.  Second, and more useful for deployment, the window is not a tuning knob that has to be re-selected as context grows.  That matters because $W$ is what bounds the drafter's cost, so a practitioner can pick the smallest window that saturates acceptance at a short calibration length and keep it.

This is a statement about acceptance, and \secref{sec:sizing} is the reason it cannot be turned into a statement about the optimal $W$ in general: neither gap is guaranteed monotone in $W$, and the cost side of the trade-off depends on the serving stack.  As with the sink sweep we report {Inferact} only, because the matching {RadixArk} sweep predates the re-basing fix of \secref{sec:rebasing} and its 1M points are not trustworthy.

\section{Gap measurement details}
\label{app:gapdetail}

\begin{table}[t]
\caption{Both gaps at matched positions, {Qwen2.5-7B-Instruct} with {EAGLE-3}, fresh-drafter law.  Intervals are 95\% document-level cluster bootstraps.}
\label{tab:gaps}
\begin{center}
\footnotesize
\begin{tabular}{rrrrrcr}
\toprule
$W$ & $\bar A$ & $\rho_{\mathrm g}(W,0)$ & $\delta$ & $[1-\bar A-\delta]_+$ & measured $\mathbb E_\mu[\alpha_t]$ & feature channel\\
\midrule
128  & .2195 & .247 & .673 & .108 & .317 \; [.261, .369] & .334\\
256  & .1823 & .206 & .675 & .143 & .319 \; [.266, .369] & .314\\
512  & .1476 & .169 & .671 & .181 & .323 \; [.268, .374] & .299\\
1024 & .1225 & .146 & .674 & .203 & .321 \; [.268, .377] & .280\\
2048 & .1013 & .121 & .680 & .219 & .317 \; [.264, .366] & .256\\
8192 & .0608 & .082 & .694 & .245 & .304 \; [.253, .354] & .218\\
\bottomrule
\end{tabular}
\end{center}
\end{table}

\textbf{Entropy tiers on the deployment law.}
Stratifying the window ablation of \secref{sec:ablations} by {SPEED-Bench}'s difficulty tiers, low-entropy prompts accept better in every arm and tier, {MAL} 1.713 against 1.634 at 32k windowed.
This agrees in direction with the corpus-position measurements below, but we do not treat it as independent confirmation, since a stratum labelled by a measured gap cannot then certify that same gap.

\textbf{Out-of-distribution evaluation.}
The {EAGLE-3} head measured in \secref{sec:e2delta} was trained on 1024-token sequences and is evaluated here at 32k mid-document, where it reaches 33--35\% argmax agreement with the truncated target against 61.8\% on its own training distribution.
This mismatch may affect the measured $\delta\approx0.68$, but the available comparisons do not establish the direction or the magnitude of that effect, so we do not present the number as the fidelity gap a purpose-built drafter would show.
The flatness in $W$ and the magnitude of the feature channel are observations under this evaluation law, and each has a reason not to be a pure artifact of the mismatch, the training length being fixed while $W$ varies over a $64\times$ range and the feature channel being a property of the architecture's input rather than of its training length; whether either persists under domain-matched deployment remains to be tested.

\textbf{Entropy tiers.}
The tier separation is stable across all six windows: $\delta$ and $\mathbb E_\mu[\alpha_t]$ are 0.715--0.735 and 0.263--0.276 on the ten high-entropy documents against 0.558--0.592 and 0.406--0.439 on the low-entropy ones.
Because the tiers are third-party labels on the benchmark rather than our own construction, this is a consistency check, not independent confirmation of the gaps.

\textbf{Greedy patterns.}
Joining the two measurements at matched positions recovers four of the five rows of Appendix~\ref{app:greedy}'s case analysis (Table~\ref{tab:patterns}); the truncated-target versus drafter comparison was not stored, so ``truncation only'' and ``all three distinct'' merge, but both are rejections and the accept/reject split remains exact.
Fidelity-only flips grow from 0.442 to 0.598 as truncation-driven rejections shrink from 0.214 to 0.070, leaving greedy agreement almost unchanged at 0.332--0.355.
The greedy fidelity flip rate is bracketed by the fidelity-only and cancellation rows below and the merged rejection row above, giving $\delta_{\mathrm g}\in[0.573,0.675]$ at $W=2048$.

\section{Measurement controls}
\label{app:controls}

Four controls in \secref{sec:setup} are load-bearing rather than hygienic, each having changed a reported quantity by more than the algorithmic effect under study.

\emph{Symmetric compiled-graph capture.}
Both arms must use the same graph-capture setting.
One arm compiled while the other runs eager moves a single speedup ratio from $0.38\times$ to $2.21\times$.

\emph{Post-load rendezvous.}
When two arms share a node, they must be synchronized after weight load, or the arms contend for host-side launch capacity and the ratio measures scheduling rather than drafting.

\emph{Profiler-free timing.}
A per-forward profiler taxes the two arms unequally, since the speculative arm issues more, smaller launches per emitted token; all reported timings are collected without one.

\emph{fp32 reductions for distributional measurements.}
All {TV} reductions in \secref{sec:gaps} are computed in fp32.
In bf16, a same-input/different-batch-shape control alone reaches {TV} $3.0\times10^{-2}$, which cannot resolve the pre-registered tolerance of 0.05; in fp32 the same control has median {TV} $3.0\times10^{-11}$ and p99 $9.6\times10^{-6}$.

\section{What the measurements do not establish}
\label{app:remaining}

\textbf{The gap measurements are not on the deployment law.}
\Secref{sec:gaps} uses corpus positions with the drafter run as a clean forward pass, while \secref{sec:e4} through \secref{sec:ablations} count reached proposal states with the actual cached drafter under eviction, re-basing, and rollback.
The existing harnesses log request-level aggregates only, with no per-round drafted token identities and no target argmax, so per-step greedy data cannot be reconstructed from what was saved.
Closing this requires one instrumented run that records, at sampled reached states through the first rejection, the full-target token, a fresh truncated-target token, and the actual cached drafter token.
That single run yields the complete five-pattern table of Appendix~\ref{app:greedy}, both greedy flip rates, agreement, the two bounds of Proposition~\ref{prop:greedy}, and the reached-stage quantities $a_i$ that Proposition~\ref{prop:block} needs.
Sweeping the same calibration grid in $W$ during that run would also supply both gaps on the deployment law, pairing with the timing already reported in \secref{sec:ablations}.
Workloads should remain separate, document-level resampling should preserve pairing and reached stages, and a paired retained-prefix intervention should be labelled as a distinct counterfactual law.

\textbf{Scope of the theory.}
Proposition~\ref{prop:span} does not apply to a drafter reading full-context features, so the index-locality bound covers the rotary argument and not the information the drafter uses.
The boundary bound of Appendix~\ref{app:ropeproof} counts positions not certified by exactness rather than observed discrepancies, and {MAL} remains a geometric model, so Proposition~\ref{prop:block} needs the stage-specific laws we do not yet have.

\textbf{Distribution shift inside the fidelity gap.}
The instrumented drafter of \secref{sec:e2delta} runs far outside its 1024-token training distribution, so $\delta$ absorbs computation-graph mismatch as well as model difference and may be affected by an amount and in a direction our comparisons do not establish.
The feature channel is a distribution-shift channel the theory does not bound either: at million-token contexts the target's hidden states are produced at sequence lengths the draft head never saw in training, and that shift lives entirely inside $\delta$.

\textbf{Greedy output equivalence.}
Under greedy decoding with a correct verifier, the speculative arm should emit a token-for-token identical stream to target-only decoding; on the {Kimi-K3} pair it did not, in any run, with divergence appearing as early as the first generated token.
Every committed token is chosen by the target's own full-context verification and acceptance is high, so the speculative arm agrees with its own verifier, and the working explanation is benign numerical nondeterminism rather than a verification defect, the two arms being separately booted engines with different decode batch shapes, one row against $K+1$.
This has not been proven.
The supporting evidence is of the right order but not decisive: the target's own batch-shape instability is median {TV} $3.0\times10^{-11}$ and p99 $9.6\times10^{-6}$ in fp32, enough to flip an argmax only where the top-two margin is minuscule.
The decisive control is cheap, one target-only rerun at matched and differing batch shapes, and it is a prerequisite for interpreting the reached-state run above rather than a follow-up to it, since logged token agreement must be reconciled against the actual verifier decision before any gap decomposition is interpreted.

\section{The full truncation grid and the threshold screen}
\label{app:grid}

Table~\ref{tab:grid} summarizes the 34-window target-only sweep of \secref{sec:e2pilot}, over the pre-registered grid $W\in\{128,256\}\cup\{512n:n=1,\ldots,32\}$.
Six of the twenty sampled documents were excluded by a content-only rule, more than 90\% of the measured tail explained by a repeating period of at most 50 tokens; these tails are benchmark padding and yielded $\hat\rho=0$ at every window, which without the filter would have manufactured an apparent floor in the low-entropy tier.
The analysis therefore uses 14 documents and 1750 positions.
The same-input/different-batch-shape control of \secref{sec:setup} has median {TV} $3.0\times10^{-11}$ and p99 $9.6\times10^{-6}$ on this data.

\textbf{Monotonicity.}
Aggregate $\hat\rho(W,0.05)$ falls almost monotonically, from 0.719 to 0.321 with a single increase of 0.006, yet individual positions frequently are not monotone: 590 of the 1347 positions that cross below $\tau=0.05$ at some window later return above it, or 43.8\%, far above the numerical floor.
A window that suffices at one position therefore need not suffice at a larger one, which is why \secref{sec:sizing} claims no monotonicity guarantee.

\textbf{Window construction is load-bearing.}
At $W\in\{512,2048,8192\}$ the median {TV} between a cache-eviction window and a fresh hard-truncated pass is 0.261, 0.211, and 0.206, with p90 of 0.766, 0.586, and 0.581.
A nominal window size does not by itself define a predictor, so the comparator must be specified as an input construction and not only as a length.

\textbf{Entropy tiers.}
The tier ordering is correct at every window but preliminary, only four clean low-entropy documents surviving the padding filter: at $W=2048$, $\hat\rho(W,0.05)$ is 0.709 on ten high-entropy documents against 0.276 on the low-entropy ones, with greedy flip rates 0.146 and 0.058.

For the teacher-forced softmax diagnostic define
\[
B(W):=\sup_{\tau\in[0,1]}(1-\hat\rho(W,\tau))(1-\tau),
\]
so that Corollary~\ref{cor:threshold} gives $\mathbb E_\mu[\alpha_t]\ge[B(W)-\delta]_+$ under the same law.
The corollary also gives the mean form $[1-\bar A-\delta]_+$, and $1-\bar A$ exceeds $B(W)$ by 0.14 to 0.27 at every window, so the threshold screen is the weaker of the two available statements and the main text reports the mean form.
The optimizing tolerance $\tau^\ast$ ranges from 0.310 at $W=128$ to 0.130 at $W=11776$, rather than the pre-registered 0.05, and Figure~\ref{fig:e2bound} plots $B(W)$ over the whole grid.

\begin{table}[h]
\caption{Target-only truncation sweep on the 14 usable 32k documents, selected windows.  Intervals are 95\% document-level cluster-bootstrap intervals on $B(W)$ (2000 resamples).}
\label{tab:grid}
\begin{center}
\small
\begin{tabular}{rrrrrrrc}
\toprule
$W$ & $\bar A$ & $1-\bar A$ & $B(W)$ & $\tau^\ast$ & $\hat\rho(W,0.05)$ & $\rho_{\mathrm g}(W,0)$ & 95\% CI on $B(W)$\\
\midrule
128   & .2195 & .7805 & .5118 & .310 & .719 & .2469 & [.476, .552]\\
256   & .1823 & .8177 & .5659 & .285 & .689 & .2057 & [.528, .601]\\
512   & .1476 & .8524 & .6178 & .265 & .652 & .1686 & [.574, .656]\\
1024  & .1225 & .8775 & .6605 & .240 & .619 & .1463 & [.620, .699]\\
2048  & .1013 & .8987 & .6944 & .215 & .585 & .1206 & [.656, .730]\\
4096  & .0772 & .9228 & .7360 & .195 & .512 & .1000 & [.706, .770]\\
8192  & .0608 & .9392 & .7779 & .155 & .448 & .0823 & [.738, .810]\\
12288 & .0512 & .9488 & .7963 & .135 & .382 & .0589 & [.759, .829]\\
16384 & .0444 & .9556 & .8121 & .135 & .321 & .0629 & [.775, .842]\\
\bottomrule
\end{tabular}
\end{center}
\end{table}

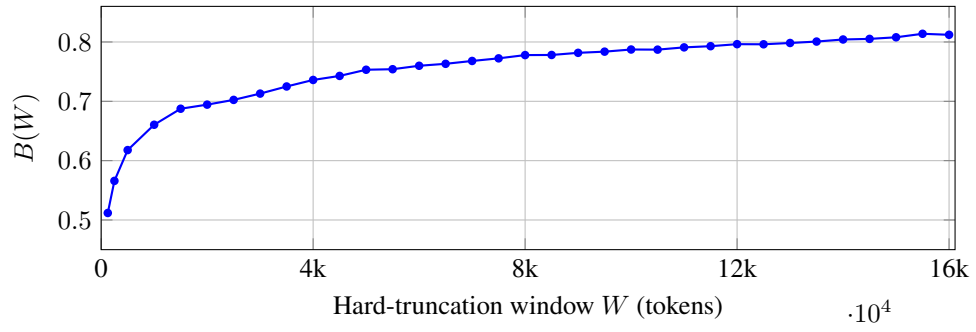
\begin{figure}[h]
\centering
\begin{tikzpicture}
\begin{axis}[
  width=.78\linewidth,height=4.8cm,
  xlabel={Hard-truncation window $W$ (tokens)},
  ylabel={$B(W)$},
  xmin=0,xmax=16500,ymin=.45,ymax=.86,
  xtick={0,4096,8192,12288,16384},
  xticklabels={0,4k,8k,12k,16k},
  ytick={.5,.6,.7,.8},
  grid=major,
]
\addplot[blue,thick,mark=*,mark size=1.2pt]
table[x=W,y=B_W,col sep=comma]{bound_by_W.csv};
\end{axis}
\end{tikzpicture}
\caption{The target-only threshold screen over all 34 windows.  It is monotone in aggregate but is the weaker of the two bounds in Corollary~\ref{cor:threshold}, and it does not measure $\delta$.}
\label{fig:e2bound}
\end{figure}

\end{document}